\documentclass[american,a4paper,runningheads,envcountsame]{llncs}
\usepackage[utf8]{inputenc}
\usepackage[T1]{fontenc}
\usepackage{babel}

\usepackage[shrink=35,kerning=true,babel=true]{microtype}
\usepackage{csquotes}
\usepackage{amsmath}
\usepackage{amssymb}
\usepackage{mathtools}
\usepackage{booktabs}

\usepackage[hidelinks]{hyperref}
\usepackage{cleveref}
\let\cref\Cref

\usepackage{tgcursor}

\usepackage{listings}
\usepackage[backend=bibtex,style=numeric-comp,doi=false,isbn=false,url=false]{biblatex}
\usepackage{graphicx}

\DeclareMathOperator{\Int}{Int}
\DeclareMathOperator{\Can}{Can}
\DeclareMathOperator{\Imp}{Imp}
\DeclareMathOperator{\Th}{Th}
\DeclareMathOperator{\dist}{dist}
\DeclareMathOperator{\Mod}{Mod}
\def\symdiff{\mathbin{\triangle}}
\DeclareMathOperator{\subsets}{\mathfrak{P}}
\DeclareMathOperator{\precision}{prec}
\DeclareMathOperator{\recall}{recall}
\renewcommand{\Pr}{\operatorname{Pr}}
\newcommand{\set}[1]{\{\,#1\,\}}
\renewcommand{\epsilon}{\varepsilon}

\usepackage[color=orange!10,linecolor=black,disable]{todonotes}

\newcommand\blfootnote[1]{%
  \begingroup
  \renewcommand\thefootnote{}\footnote{#1}%
  \addtocounter{footnote}{-1}%
  \endgroup
}

\begin{document}

\title{On the Usability of\\ Probably Approximately Correct Implication Bases}
\author{Daniel Borchmann\inst{1} \and Tom Hanika\inst{2,3} \and
  Sergei Obiedkov\inst{4}} \institute{%
  Chair of Automata Theory\\
  Technische Universität Dresden, Germany\\
  \and
  Knowledge \& Data Engineering Group\\
  University of Kassel, Germany\\
  \and
  Interdisciplinary Research Center for Information System Design\\
  University of Kassel, Germany\\
  \and
  National Research University
  Higher School of Economics, Moscow, Russia\\[1ex]
  \email{daniel.borchmann@tu-dresden.de,
    tom.hanika@cs.uni-kassel.de,sergei.obj@gmail.com}
}

\maketitle

\blfootnote{The authors of this work are given in alphabetical order.
  No priority in authorship is implied.}

\begin{abstract}
  We revisit the notion of \emph{probably approximately correct
    implication bases} from the literature and present a first
  formulation in the language of formal concept analysis, with the
  goal to investigate whether such bases represent a suitable
  substitute for exact implication bases in practical use-cases.  To
  this end, we quantitatively examine the behavior of probably
  approximately correct implication bases on artificial and real-world
  data sets and compare their precision and recall with respect to
  their corresponding exact implication bases.  Using a small example,
  we also provide qualitative insight that implications from probably
  approximately correct bases can still represent meaningful knowledge
  from a given data set.
\end{abstract}


\section{Introduction}
\label{sec:introduction}

From a practical point of view, computing implication bases of formal
contexts is a challenging task.  The reason for this is twofold: on
the one hand, bases of formal contexts can be of exponential
size~\cite{DBLP:journals/jucs/Kuznetsov04} (see also an earlier work \cite{Kautz_1995} for the same result presented in different terms), and thus just writing out
the result can take a long time.  On the other hand, even in cases
where implication bases can be small, efficient methods to compute
them are unknown in general, and running times may thus be much higher
than necessary.  This is particularly true for computing the
\emph{canonical basis}, where only very few
algorithms~\cite{DBLP:conf/icfca/Ganter10,DBLP:journals/amai/ObiedkovD07}
are known, which all in addition to the canonical basis have to compute
the complete concept lattice.

Approaches to tackle this problem are to parallelize existing
algorithms~\cite{KrBo-CLA15}, or restrict attention to implication
bases that are more amenable to algorithmic treatment, such as proper
premises~\cite{RyDiBo-AMAI13} or D-bases~\cite{Adaricheva_2017}.  The
latter usually comes with the downside that the number of implications
is larger than necessary.  A further, rather pragmatic approach is to
consider implications as strong association rules and employ highly
optimized association rule miners, but then the number of resulting
implications increases even more.

In this work, we want to introduce another approach, which is
conceptually different from those previously mentioned: instead of
computing exact bases that can be astronomically large and hard to
compute, we propose to compute \emph{approximately correct} bases that
capture essential parts of the implication theory of the given data
set, and that are easier to obtain.  To facilitate algorithmic
amenability, it turns out to be a favorable idea to compute bases that
are approximately correct \emph{with high probability}.  Those bases
are called \emph{probably approximately correct bases} (PAC~bases),
and they can be computed in polynomial time.

PAC~bases allow to relax the rather strong condition of computing an
exact representation of the implicational knowledge of a data set.
However, this new freedom comes at the price of the uncertainty that
approximation always brings: is the result suitable for my intended
application?  Of course, the answer to this questions depends deeply
on the application in mind and cannot be given in general.  On the
other hand, some general aspects of the usability of PAC~bases can be
investigated, and it is the purpose of this work to provide first
evidence that such bases can indeed be useful.  More precisely, we
want to show that despite the probabilistic nature of these bases, the
results they provide are indeed not significantly different from the
actual bases (in a certain sense that we shall make clear later), and
that the returned implication sets can contain meaningful
implications. To this end, we investigate PAC~bases on both artificial
and real-world data sets and discuss their relationships with their
exact counterparts.

The idea of considering PAC~bases is not new~\cite{Kautz_1995}, but
has somehow not received much attention as a different, and maybe
tantamount, approach to extract implicational knowledge from formal
contexts.  Moreover, PAC~bases also allow interesting
connections between formal concept analysis and \emph{query
  learning}~\cite{Angluin_1988} (as we shall see), a connection that
with respect to attribute exploration awaits further investigation.

The paper is structured as follows.  After a brief review of related
work in Section~\ref{sec:related-work}, we shall introduce probably
approximately correct bases in Section~\ref{sec:comp-appr-bases},
including a means to compute them based on results from query
learning.  In Section~\ref{sec:usability}, we discuss usability
issues, both from a quantitative and a qualitative point of view.  We
shall close our discussion with summary and outlook in
Section~\ref{sec:summary-outlook}.

\section{Related Work}
\label{sec:related-work}

Approximately learning concepts with high probability has first been
introduced in the seminal work by Valiant~\cite{valiant84learnable}.
From this starting point, \emph{probably approximately correct}
learning has come a long way and has been applied in a variety of
use-cases.  Work that is particularly relevant for our concerns is by
\citeauthor{Kautz_1995} on Horn approximation of empirical
data~\cite{Kautz_1995}. In there a first algorithm for computing
probably approximately correct implication bases for a given data set
has been proposed \cite[Theorem~15]{Kautz_1995}. This algorithm has
the benefit that all closed sets of the actual implication theory will
be among the ones of the computed theory, but the latter may possibly
contain more.  However, the algorithm requires direct access to the
actual data, which therefore must be given explicitly.

Approximately correct bases have also been considered before in the
realm of formal concept analysis, although not much.  The dissertation
by Babin~\cite{Babin12} contains results about \emph{approximate
  bases} and some first experimental evaluations.  However, this
notion of approximation is different from the one we want to employ in
this work: Babin defines a set of implications $\mathcal{H}$ to be an
approximation of a given set $\mathcal{L}$ if the closure operators of
$\mathcal{L}$ and $\mathcal{H}$ coincide on most sets.  In our work,
$\mathcal{H}$ is an approximation of $\mathcal{L}$ if and only if the
number of models in which $\mathcal{H}$ and $\mathcal{L}$ differ is
small.  Details will follow in the next section.  The approach of
considering implications with \emph{high confidence} in addition to
exact implications can also be seen as a variant of approximate
bases~\cite{Borchmann14}.

To compute PAC bases, we shall make use of results from the research
field of \emph{query learning}~\cite{Angluin_1988}.  More precisely,
we shall make use of the work by \citeauthor{Angluin_1992} on learning
Horn theories through query learning~\cite{Angluin_1992}, where the
target Horn theory is accessible only through a \emph{membership} and
an \emph{equivalence} oracle.  Using existing results, this algorithm
can easily be adapted to compute probably approximately correct Horn
theories, and we shall give a self-contained explanation of the
algorithm in this work. Related to query learning is \emph{attribute
  exploration}~\cite{DBLP:books/sp/GanterO16}, an algorithm from
formal concept analysis that allows to learn Horn theories from
domain experts.

\section{Probably Approximately Correct Bases from Query Learning}
\label{sec:comp-appr-bases}

Before introducing approximately correct and probably approximately
correct bases in Section~\ref{sec:approximate-bases}, we shall first
give a brief (and dense) recall in
Section~\ref{sec:bases-implications} of the relevant definitions and
terminologies from formal concept analysis used in this work.  We then
demonstrate in Section~\ref{sec:how-compute-probably} how probably
approximately correct bases can be computed using ideas from query
learning.

\subsection{Bases of Implications}
\label{sec:bases-implications}

Recall that a formal context is just a triple $\mathbb K = (G,M,I)$
where $G$ and $M$ are sets and $I \subseteq G \times M$.  We shall
denote the derivation operators in $\mathbb K$ with the usual
$\cdot'$-notation, i.e., $A' = \set{ m \in M \mid \forall g \in
  A\colon (g,m) \in I}$ and $B' = \set{ g \in G \mid \forall m \in
  B\colon (g,m) \in I}$ for $A \subseteq G$ and $B \subseteq M$.  The
sets $A$ and $B$ are \emph{closed} in $\mathbb K$ if $A = A''$ and $B
= B''$, respectively.  The set of subsets of $M$ closed in $\mathbb K$
is called the set of \emph{intents} of $\mathbb K$ and is denoted by
$\mathbb \Int(\mathbb K)$.

An \emph{implication} over $M$ is an expression $X \to Y$ where $X, Y
\subseteq M$.  The set of all implications over $M$ is denoted by
$\Imp(M)$.  A set $A \subseteq M$ is \emph{closed} under $X \to Y$ if
$X \not\subseteq A$ or $Y \subseteq A$.  In this case, $A$ is also
called a \emph{model} of $X \to Y$ and $X \to Y$ is said to \emph{respect} $A$.  The set $A$ is \emph{closed}
under a set of implications $\mathcal{L}$ if $A$ is closed under every
implication in $\mathcal{L}$.  The set of all sets closed under
$\mathcal{L}$, the \emph{models} of $\mathcal{L}$, is denoted by
$\Mod(\mathcal{L})$.

The implication $X \to Y$ is \emph{valid} in $\mathbb K$ if $\{g\}'$ is
closed under $X \to Y$ for all $g \in G$ (equivalently: $X' \subseteq
Y'$, $Y \subseteq X''$).  A set $\mathcal{L}$ of implications is
\emph{valid} in $\mathbb K$ if every implication in $\mathcal{L}$ is
valid in $\mathbb K$.  The set of all implications valid in $\mathbb
K$ is the \emph{theory} of $\mathbb K$, denoted by $\Th(\mathbb K)$.
Clearly, $\Mod(\Th(\mathbb K)) = \Int(\mathbb K)$.

Let $\mathcal{L} \subseteq \Imp(M)$ and $(X \to Y) \in \Imp(M)$.  We
say that $X \to Y$ \emph{follows} from $\mathcal{L}$, written
$\mathcal{L} \models (X \to Y)$, if for all contexts $\mathbb L$ where
$\mathcal{L}$ is valid, $X \to Y$ is valid as well.  Equivalently,
$\mathcal{L} \models (X \to Y)$ if and only if $Y \subseteq
\mathcal{L}(X)$, where $\mathcal{L}(X)$ is the $\subseteq$-smallest
superset of $X$ that is closed under all implications from
$\mathcal{L}$.

A set $\mathcal{L} \subseteq \Imp(M)$ is an \emph{exact implication
  basis} (or simply \emph{basis}) of $\mathbb K$ if $\mathcal{L}$ is
\emph{sound} and \emph{complete} for $\mathbb K$.  Here the set
$\mathcal{L}$ is \emph{sound} for $\mathbb K$ if it is valid in
$\mathbb K$.  Dually, $\mathcal{L}$ is \emph{complete} for $\mathbb K$
if every implication valid in $\mathbb K$ follows from $\mathcal{L}$.
Alternatively, $\mathcal{L}$ is a basis of $\mathbb K$ if the models of
$\mathcal{L}$ are the intents of $\mathbb K$.

A basis $\mathcal{L}$ of $\mathbb K$ is called \emph{irredundant} if
no strict subset of $\mathcal{L}$ is a basis of $\mathbb K$.  A basis
$\mathcal{L}$ of $\mathbb K$ is called \emph{minimal} if there does
not exist another basis $\hat{\mathcal{L}}$ of $\mathbb K$ with
strictly fewer elements, i.e., with $\lvert \hat{\mathcal{L}} \rvert <
\lvert \mathcal{L} \rvert$.  Every minimal basis is clearly
irredundant, but the converse is not true in general.

For the case of finite contexts $\mathbb K = (G,M,I)$, i.e., where
both $G$ and $M$ are finite, a minimal basis can be given explicitly as
the so-called \emph{canonical basis} of $\mathbb K$ \cite{guigues1986famille}.  For this, recall
that a \emph{pseudo-intent} of $\mathbb K$ is a set $P \subseteq M$
such that $P \neq P''$ and each pseudo-intent $Q \subsetneq P$
satisfies $Q'' \subseteq P$.  The canonical basis is defined as $\Can(\mathbb K)
= \set{ P \to P'' \mid P \text{ pseudo-intent of } \mathbb K}$.  It is
well known that $\Can(\mathbb K)$ is a minimal basis of $\mathbb K$.

\subsection{Probably Approximately Correct Implication Bases}
\label{sec:approximate-bases}

Exact implication bases, in particular irredundant or minimal ones,
provide a convenient way to represent the theory of a formal context
in a compact way.  However, the computation of such bases is -- not
surprisingly -- difficult, and currently known algorithms impose an
enormous additional overhead on the already high running times.  On
the other hand, data sets originating from real-world data are usually
\emph{noisy}, i.e., contain errors and inaccuracies, and computing
exact implication bases of such data sets is futile from the very
beginning: in these cases, it is sufficient to compute
an \emph{approximation} of the exact implication basis.  The only
thing one has to make sure is that such bases have a controllable
error lest they be unusable.

More formally, instead of computing exact implication bases of finite
contexts $\mathbb K$, we shall consider \emph{approximately correct
  implication bases} of $\mathbb K$, hoping that such approximations
still capture essential parts of the theory of $\mathbb K$, while being
easier to compute.  Clearly, the precise notion of approximation
determines the usefulness of this approach.  In this work, we want to
take the stance that a set $\mathcal{H}$ of implications is an
\emph{approximately correct basis} of $\mathbb K$ if the closed sets of
$\mathcal{H}$ are \enquote{most often} closed in $\mathbb K$ and vice
versa.  This is formalized in the following definition.

\begin{definition}
  Let $M$ be a finite set and let $\mathbb K = (G,M,I)$ be a formal
  context.  A set $\mathcal{H} \subseteq \Imp(M)$ is called an
  \emph{approximately correct basis} of $\mathbb K$ with
  \emph{accuracy} $\epsilon > 0$ if
  \begin{equation*}
    \dist(\mathcal{H},\mathbb K) \coloneqq \frac {\lvert
      \Mod(\mathcal{H}) \symdiff \Int(\mathbb K)\rvert} {2^{\lvert M
        \rvert}} < \epsilon.
  \end{equation*}
  We call $\dist(\mathcal{H}, \mathbb K)$ the \emph{Horn-distance}
  between $\mathcal{H}$ and $\mathbb K$.
\end{definition}

The notion of Horn-distance can easily be extended to sets of
implications: the \emph{Horn-distance} between $\mathcal{L} \subseteq
\Imp(M)$ and $\mathcal{H} \subseteq \Imp(M)$ is defined as in the
definition above, replacing $\Int(\mathbb K)$ by $\Mod(\mathcal{L})$.
Note that with this definition, $\dist(\mathcal{L}, \mathbb K) =
\dist(\mathcal{L}, \mathcal{H})$ for every exact implication basis
$\mathcal{H}$ of $\mathbb K$.  On the other hand, every set
$\mathcal{L}$ can be represented as a basis of a formal context
$\mathbb K$, and, in this case, $\dist(\mathcal{H},\mathcal{L}) =
\dist(\mathcal{H},\mathbb K)$ for all $\mathcal{H} \subseteq \Imp(M)$.

For practical purposes, it may be enough to be able to compute
approximately correct bases with high probability.  This eases algorithmic
treatment from a theoretical perspective, in the sense that it is
possible to find algorithms that run in polynomial time.

\begin{definition}
  Let $M$ be a finite set and let $\mathbb K = (G,M,I)$ be a formal
  context.  Let $\Omega = (W,\mathcal{E}, \Pr)$ be a probability
  space.  A random variable $\mathcal{H}\colon \Omega \to
  \subsets(\Imp(M))$ is called a \emph{probably approximately correct
    basis} (PAC~basis) of $\mathbb K$ with \emph{accuracy} $\epsilon >
  0$ and \emph{confidence} $\delta > 0$ if
  $\Pr(\dist(\mathcal{H},\mathbb K) > \epsilon) < \delta$.
\end{definition}

\subsection{How to Compute Probably Approximately Correct Bases}
\label{sec:how-compute-probably}

We shall make use of query learning to compute PAC~bases.  The
principal goal of query learning is to find explicit representation of
\emph{concepts} under the restriction of only having access to certain
kinds of \emph{oracles}.  The particular case we are interested in is
to learn conjunctive normal forms of Horn formulas from
\emph{membership} and \emph{equivalence} oracles.  Since conjunctive
normal forms of Horn formulas correspond to sets of unit implications,
this use-case allows learning sets of implications from oracles.
Indeed, the restriction to unit implications can be dropped, as we
shall see shortly.

Let $\mathcal{L} \subseteq \Imp(M)$ be a set of implications.  A
\emph{membership oracle} for $\mathcal{L}$ is a function $f \colon
\subsets(M) \to \{\top,\bot\}$ such that $f(X) = \top$ for $X
\subseteq M$ if and only if $X$ is a model of $\mathcal{L}$.  An
\emph{equivalence oracle} for $\mathcal{L}$ is a function $g \colon
\subsets(\Imp(M)) \to \{\top\} \cup \subsets(M)$ such that
$g(\mathcal{H}) = \top$ if and only if $\mathcal{H}$ is equivalent to
$\mathcal{L}$, i.e., $\Mod(\mathcal{H}) = \Mod(\mathcal{L})$.
Otherwise, $X \coloneqq g(\mathcal{H})$ is a \emph{counterexample} for
the equivalence of $\mathcal{H}$ and $\mathcal{L}$, i.e., $X \in
\Mod(\mathcal{H}) \symdiff \Mod(\mathcal{L})$.  We shall call $X$ a
\emph{positive counterexample} if $X \in
\Mod(\mathcal{L})\setminus\Mod(\mathcal{H})$, and a \emph{negative
  counterexample} if $X \in
\Mod(\mathcal{H})\setminus\Mod(\mathcal{L})$.

To learn sets of implications through membership and equivalence
oracles, we shall use the well-known HORN1 algorithm~\cite{Angluin_1992}.
Pseudocode describing this algorithm is given in
Figure~\ref{fig:horn1-in-fca}, where we have adapted the algorithm to
use FCA terminology.

\begin{figure}[tp]
  \centering
\begin{lstlisting}
define horn1($M$,member?,equivalent?)
  $\mathcal{H}$ := $\emptyset$
  while $C$ := equivalent?($\mathcal{H}$) is a counterexample do
    if some $A \to B \in \mathcal{H}$ does not respect $C$ then
      replace all implications $A \to B \in \mathcal{H}$
        not respecting $C$ by $A \to B \cap C$
    else
      find first $A \to B \in \mathcal{H}$ such that
        $C \cap A \neq A$ and member?($C \cap A$) returns false
      if $A \to B$ exists then
        replace $A \to B$ by $C \cap A \to B \cup (A \setminus C)$
      else
        add $C \to M$ to $\mathcal{H}$
      end
    end
  end
  return $\mathcal{H}$
end
\end{lstlisting}
  \caption{HORN1, adapted to FCA~terminology}
  \label{fig:horn1-in-fca}
\end{figure}

The principal way the HORN1 algorithm works is the following: keeping
a \emph{working hypothesis} $\mathcal{H}$, the algorithm repeatedly
queries the equivalence oracle about whether $\mathcal{H}$ is
equivalent to the sought basis $\mathcal{L}$.  If this is the case, the
algorithm stops.  Otherwise, it receives a counterexample $C$ from the
oracle, and depending on whether $C$ is a positive or a negative
counterexample, it adapts the hypothesis accordingly.  In the case $C$
is a positive counterexample, all implications in $\mathcal{H}$ not
respecting $C$ are modified by removing attributes not in $C$ from their conclusions.  Otherwise, $C$ is a negative
counterexample, and $\mathcal{H}$ must be adapted so that $C$ is not
a model of $\mathcal{H}$ anymore.  This is done by searching for an
implication $(A \to B) \in \mathcal{H}$ such that $C \cap A \neq
A$ is not a model of $\mathcal{L}$, employing the membership query.  If
such an implication is found, it is replaced by
$C \cap A \to B \cup (A\setminus C)$.  Otherwise, the implication $C \to M$
is simply added to $\mathcal{H}$.

With this algorithm, it is possible to learn implicational theories
from equivalence and membership oracles alone.  Indeed, the resulting
set $\mathcal{H}$ of implications is always the canonical basis equivalent
to $\mathcal{L}$~\cite{Arias_2011}.  Moreover, the algorithm always
runs in polynomial time in $\lvert M \rvert$ and  the size of the sought implication
basis~\cite[Theorem~2]{Angluin_1992}.

We now want to describe an adaption of the HORN1 algorithm that allows
to compute PAC~bases in polynomial time in size of $M$, the output
$\mathcal{L}$, as well as $1/\epsilon$ and $1/\delta$.  For this we
modify the original algorithm of Figure~\ref{fig:horn1-in-fca} as
follows: given a set $\mathcal{H}$ of implications, instead of
checking exactly whether $\mathcal{H}$ is equivalent to the sought
implicational theory $\mathcal{L}$, we employ the strategy of
\emph{sampling}~\cite{Angluin_1988} to simulate the equivalence
oracle.  More precisely, we sample for a certain number of iterations
subsets $X$ of $M$ and check whether $X$ is a model of $\mathcal{H}$
and not of $\mathcal{L}$ or vice versa.  In other words, we ask
whether $X$ is an element of $\Mod(\mathcal{H}) \symdiff
\Mod(\mathcal{L})$.  Intuitively, given enough iterations, the
sampling version of the equivalence oracle should be close to the
actual equivalence oracle, and the modified algorithm should return a
basis that is close to the sought one.

\begin{figure}[tp]
  \centering
\begin{lstlisting}
define approx-equivalent?(member?,$\epsilon$,$\delta$)
  $i$ := 0 ;; number of equivalence queries

  return function($\mathcal{H}$) begin
      $i$ := $i$ + 1
      for $\ell_{i}$ times do
        choose $X \subseteq M$
        if (member?($X$) and $X \not\in \Mod(\mathcal{H})$) or
           (not member?($X$) and $X \in \Mod(\mathcal{H})$) then
          return $X$
        end
      end
      return true
    end
end

define pac-basis($M$,member?,$\epsilon$,$\delta$)
  return horn1(M,member?,approx-equivalent?(member?,$\epsilon$,$\delta$))
end
\end{lstlisting}
  \caption{Computing PAC bases}
  \label{fig:pac-bases-from-horn1}
\end{figure}

Pseudocode implementing the previous elaboration is given in
Figure~\ref{fig:pac-bases-from-horn1}, and it requires some further
explanation.  The algorithm computing a PAC~basis of an
implication theory given by access to a membership oracle is called
\lstinline{pac-basis}.  This function is implemented in terms of
\lstinline{horn1}, which, as explained before, receives as equivalence
oracle a sampling algorithm that uses the membership oracle to decide
whether a randomly sampled subset is a counterexample.  This sampling
equivalence oracle is returned by \lstinline{approx-equivalent?}, and
manages an internal counter $i$ keeping track of the number of
invocations of the returned equivalence oracle.  Every time this
oracle is called, the counter is incremented and thus influences the
number $\ell_{i}$ of samples the oracle draws.

The question now is whether the parameters $\ell_{i}$ can be chosen
so that \lstinline{pac-basis} computes a PAC~basis in every case.
The following theorem gives an affirmative answer.

\begin{theorem}
  \label{thm:correctness-of-pac-basis}
  Let $0 < \epsilon \le 1$ and $0 < \delta \leq 1$.  Set
  \begin{equation*}
    \ell_{i} \coloneqq \left\lceil \frac1\epsilon \cdot \left(i -
        \log_{2} \delta\right)\right\rceil.
  \end{equation*}
  Denote with $\mathcal{H}$ the random variable representing the
  outcome of the call to \lstinline{pac-basis} with arguments $M$ and
  the membership oracle of $\mathcal{L}$.  Then $\mathcal{H}$ is a
  PAC~basis for $\mathcal{L}$, i.e., $\Pr\left(\dist(\mathcal{H},
    \mathcal{L}) > \epsilon \right) < \delta$, where $\Pr$ denotes the
  probability distribution over all possible runs of
  \lstinline{pac-basis} with the given arguments.  Moreover,
  \lstinline{pac-basis} finishes in time polynomial in $\lvert
  M\rvert$, $\lvert \mathcal{L}\rvert$, $1/\epsilon$, and $1/\delta$.
\end{theorem}
\begin{proof}
  We know that the runtime of \lstinline{pac-basis} is bounded by a
  polynomial in the given parameters, provided we count the
  invocations of the oracles as single steps.  Moreover, the numbers
  $\ell_{i}$ are polynomial in $\lvert M\rvert$, $\lvert
  \mathcal{L}\rvert$, $1/\epsilon$, and $1/\delta$ (since $i$ is polynomial in $\lvert M\rvert$ and $\lvert \mathcal{L}\rvert$), and thus
  \lstinline{pac-basis} always runs in polynomial time.

  The algorithm \lstinline{horn1} requires a number of counterexamples polynomial in $\lvert M\rvert$ and
  $\lvert \mathcal{L}\rvert$. Suppose that this number is at most $k$.  We want to ensure that in $i$th call to
  the sampling equivalence oracle, the probability $\delta_{i}$ of
  failing to find a counterexample (if one exists) is at most $\delta/2^{i}$.  Then
  the probability of failing to find a counterexample in any of at most $k$
  calls to the sampling equivalence oracle is at most
  \begin{equation*}
    \frac\delta2 + \left(1-\frac\delta2\right)\cdot\left(\frac\delta4 +
      \left(1-\frac\delta4\right)\left(\frac\delta8 +
        \left(1-\frac\delta8\right)\cdot\biggl(\dots \biggr)\right)\right)
    \le \frac\delta2 + \frac\delta4 + \dots + \frac\delta{2^{k}} < \delta.
  \end{equation*}

  Assume that in some step $i$ of the algorithm, the currently
  computed hypothesis $\hat{\mathcal{H}}$ satisfies
  \begin{equation}
    \label{eq:1}
    \dist(\hat{\mathcal{H}},\mathcal{L}) = \frac{\lvert
      \Mod\hat{\mathcal{H}} \symdiff \Mod\mathcal{L}\rvert}{2^{\lvert
        M\rvert}} > \epsilon.
  \end{equation}
  Then choosing $X \in \Mod\hat{\mathcal{H}} \symdiff \Mod\mathcal{L}$
  succeeds with probability at least $\epsilon$, and the probability
  of failing to find a counterexample in $\ell_{i}$ iterations is at
  most $(1-\epsilon)^{\ell_{i}}$. We want to choose $\ell_{i}$
  such that $(1-\epsilon)^{\ell_{i}} < \delta_{i}$.  We obtain
  \begin{equation*}
    \log_{1-\epsilon}\delta_{i} =
    \frac{\log_{2}\delta_{i}}{\log_{2}(1-\epsilon)} =
    \frac{\log_{2}(1/\delta_{i})}{-\log_{2}(1-\epsilon)} \leq
    \frac{\log_{2}(1/\delta_{i})}{\epsilon},
  \end{equation*}
  because $-\log_{2}(1-\epsilon) > \epsilon$.  Thus, choosing any
  $\ell_{i}$ satisfying $\ell_{i} >
  \frac1\epsilon\log_{2}\frac1{\delta_{i}}$ is sufficient
  for our algorithm to be approximately correct.  In particular, we can set
  \begin{equation*}
    \ell_{i} \coloneqq \left\lceil \frac1\epsilon
      \log_{2}\frac1{\delta_{i}}\right\rceil =
    \left\lceil\frac1\epsilon
      \log_{2}\frac{2^{i}}{\delta} \right\rceil =
    \left\lceil \frac1\epsilon \left(i - \log_{2}\delta\right) \right\rceil,
  \end{equation*}
  as claimed.  This finishes the proof.
\end{proof}

The preceding argumentation relies on the fact that we choose subsets
$X \subseteq M$ uniformly at random.  However, it is conceivable that,
for certain applications, computing PAC~bases for uniformly sampled
subsets $X \subseteq M$ might be too much of a restriction, in
particular, when certain combinations of attributes are more likely
than others.  In this case, PAC~bases are sought with respect to some
\emph{arbitrary distribution} of $X \subseteq M$.

It turns out that such a generalization of
Theorem~\ref{thm:correctness-of-pac-basis} can easily be obtained.  For
this, we observe that the only place where uniform sampling is needed
is in Equation~\eqref{eq:1} and the subsequent argument that choosing
a counterexample $X \in \Mod(\mathcal{\hat{H}}) \symdiff
\Mod(\mathcal{L})$ succeeds with probability at least $\epsilon$.

To generalize this to an arbitrary distribution, let $X$ be a random
variable with values in $\subsets(M)$, and denote the corresponding
probability distribution with $\Pr_{1}$.  Then Equation~\eqref{eq:1}
can be generalized to
\begin{equation*}
  \Pr_{1}(X \in \Mod(\hat{\mathcal{H}}) \symdiff \Mod(\mathcal{L})) > \epsilon.
\end{equation*}
Under this condition, choosing a counterexample in
$\Mod(\hat{\mathcal{H}}) \symdiff \Mod(\mathcal{L})$ still succeeds
with probability at least $\epsilon$, and the rest of the proof goes
through.  More precisely, we obtain the following result.

\begin{theorem}
  Let $M$ be a finite set, $\mathcal{L} \subseteq \Imp(M)$.  Denote
  with $X$ a random variable taking subsets of $M$ as values, and let
  $\Pr_{1}$ be the corresponding probability distribution.  Further
  denote with $\mathcal{H}$ the random variable representing the
  results of \lstinline{pac-basis} when called with arguments $M$,
  $\epsilon > 0$, $\delta > 0$, a membership oracle for $M$, and where
  the sampling equivalence oracle uses the random variable $X$ to draw
  counterexamples.  If $\Pr_{2}$ denotes the corresponding probability
  distribution for $\mathcal{H}$, then
  \begin{equation*}
    \Pr_{2}\Bigl( \Pr_{1}\bigl(X \in \Mod(\mathcal{H}) \symdiff
        \Mod(\mathcal{L})\bigr) > \epsilon \Bigr) < \delta.
  \end{equation*}
  Moreover, the runtime of \lstinline{pac-basis} is bounded by a
  polynomial in the sizes of $M$, $\mathcal{L}$ and the values
  $1/\epsilon$, $1/\delta$.
\end{theorem}

\section{Usability}
\label{sec:usability}

We have seen that PAC~bases can be computed fast, but the question
remains whether they are a useful representation of the implicational
knowledge embedded in a given data set.  To approach this question, we
now want to provide a first assessment of the usability in terms of
quality and quantity of the approximated implications. To this end, we
conduct several experiments on artificial and real-world data sets.
In~\cref{sec:how-much-worse}, we measure the approximation quality
provided by PAC~bases.  Furthermore, in~\cref{sec:how-much-different}
we examine a particular context and argue that PAC~bases also provide
a meaningful approximation of the corresponding canonical basis.

\subsection{Practical Quality of Approximation}
\label{sec:how-much-worse}

In theory, PAC~bases provide good approximation of exact bases with
high probability. But how do they behave with respect to practical
situations?  To give first impressions on the answer to this question,
we shall investigate three different data sets.  First, we examine how
the \lstinline{pac-basis} algorithm performs on real-world formal
contexts.  For this we utilize a data set based on a public data-dump
of the BibSonomy platform, as described
in~\cite{conf/cla/BorchmannH16}. Our second experiment is conducted on
a subclass of artificial formal contexts.  As it was shown
in~\cite{conf/cla/BorchmannH16}, it is so far unknown how to generate
genuine random formal contexts.  Hence, we use the \enquote{usual way}
of creating artificial formal contexts, with all warnings in place:
for a given number of attributes and density, choose randomly a valid
number of objects and use a biased coin to draw the crosses.  The last
experiment is focused on repetition stability: we calculate PAC bases
of a fixed formal context multiple times and examine the standard
deviation of the results.

The comparison will utilize three different measures. For every
context in consideration, we shall compute the Horn-distance between
the canonical basis and the approximating bases returned by
\lstinline{pac-basis}.  Furthermore, we shall also make use of the
usual \emph{precision} and \emph{recall} measures, defined as follows.

\begin{definition}
  Let $M$ be a finite set and let $\mathbb{K}= (G,M,I)$ be a formal
  context.  Then the \emph{precision} and \emph{recall} of
  $\mathcal{H}$, respectively, are defined as
  \begin{align*}
  \precision(\mathbb{K},\mathcal{H})&\coloneqq \frac{|\{(A\to B)\in
      \mathcal{H}\mid \Can(\mathbb K)\models (A\to B)\}|}{|\mathcal{H}|},\\
  \recall(\mathbb{K},\mathcal{H})&\coloneqq \frac{|\{(A\to B)\in
      \Can(\mathbb K)\mid \mathcal{H}\models (A\to B)\}|}{|\Can(\mathbb K)|}.
  \end{align*}
\end{definition}

In other words, precision is measuring the fraction of valid
implications in the approximating basis $\mathcal{H}$, and recall is
measuring the fraction of valid implications in the canonical basis
that follow semantically from the approximating basis $\mathcal{H}$.
Since we compute precision and recall for multiple contexts in the
experiments, we consider the \emph{macro average} of those measures,
i.e., the mean of the values of these measure on the given contexts.

\subsubsection{BibSonomy Contexts}
\label{sec:bibsonomy-contexts}

\begin{figure}[t]
  \centering
  \includegraphics[width=0.32\textwidth]{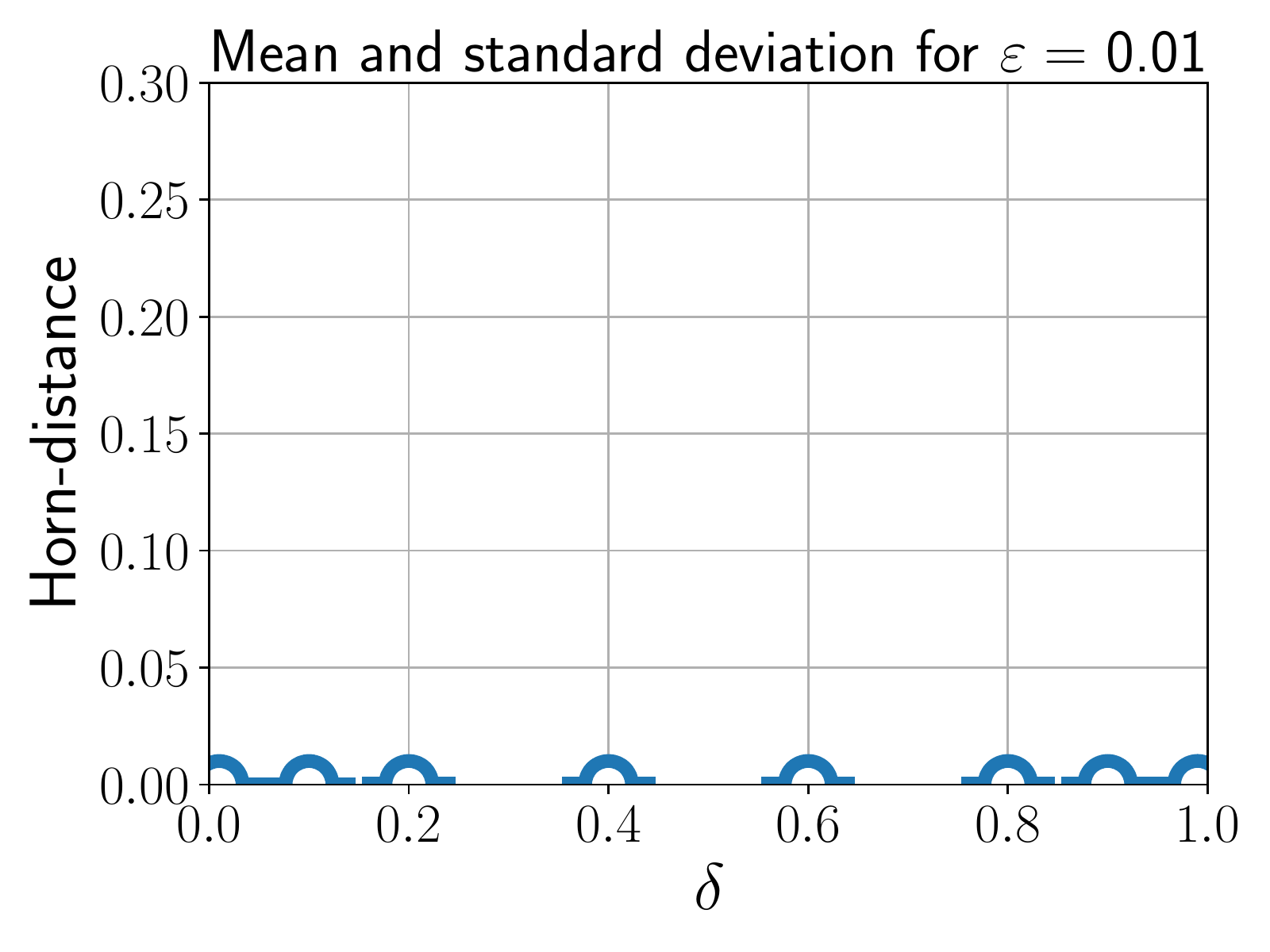}
  \includegraphics[width=0.32\textwidth]{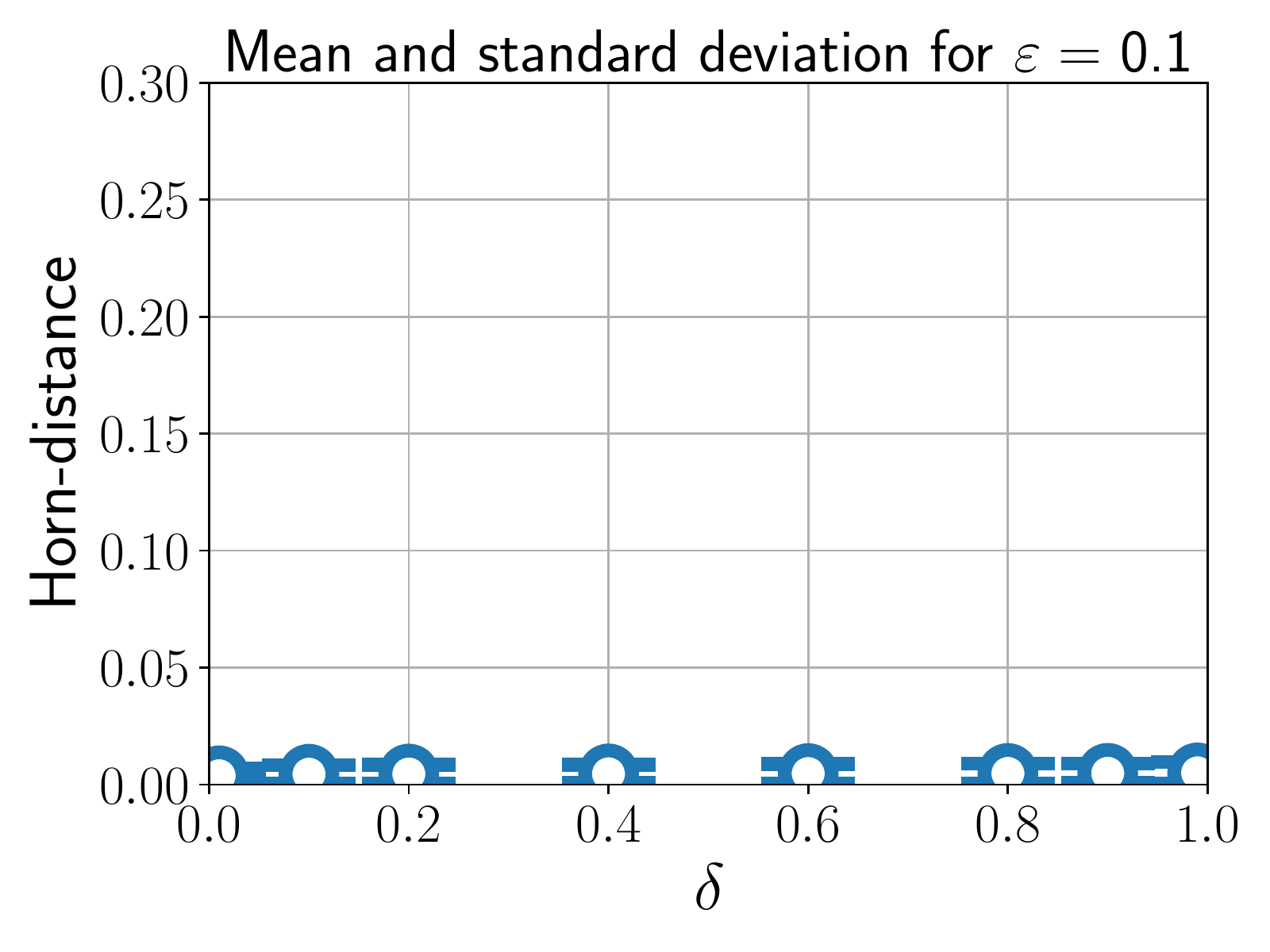}
  \includegraphics[width=0.32\textwidth]{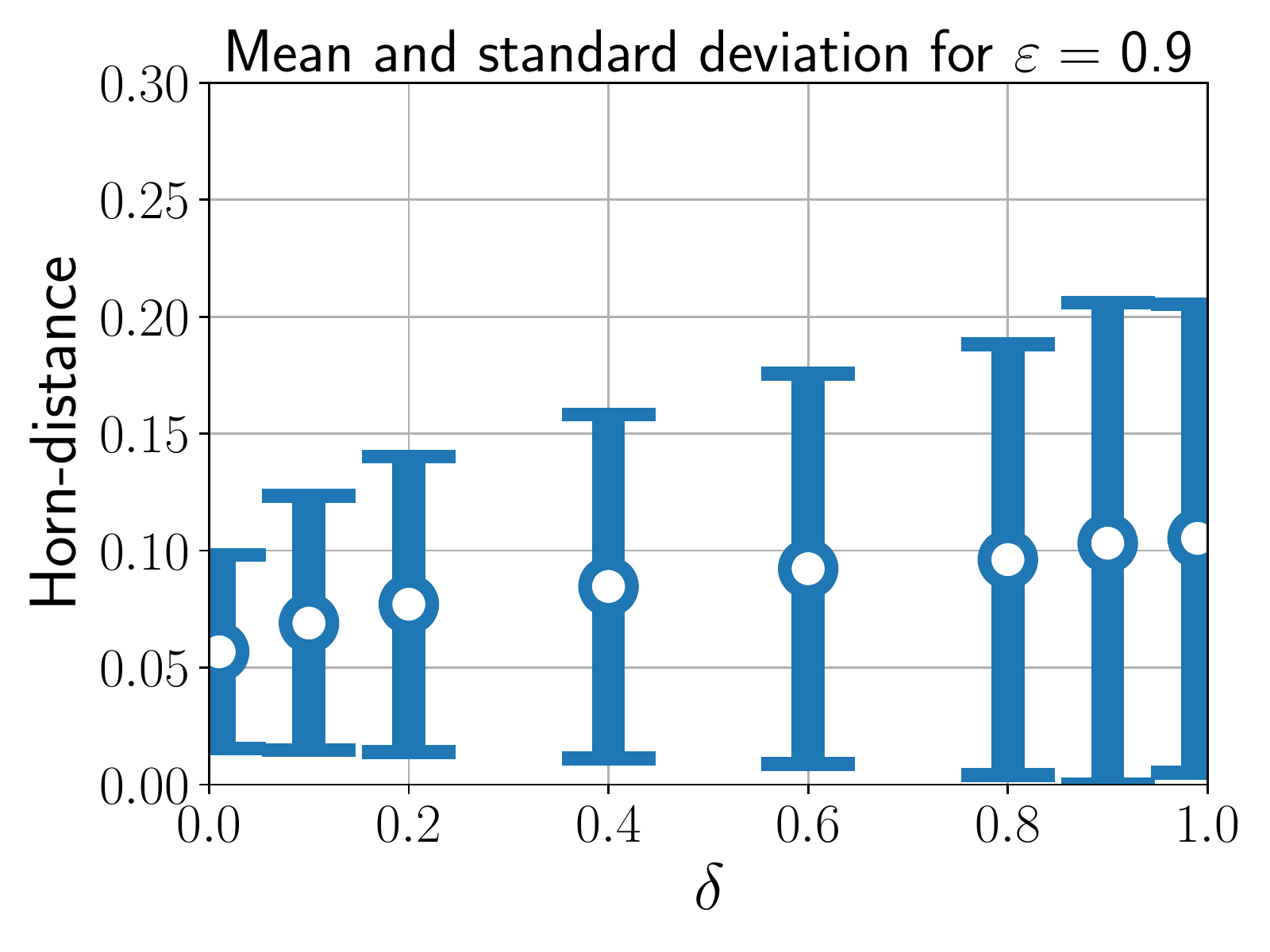}
  \caption{Horn-distances between the contexts from the BibSonomy data
    set and corresponding PAC~bases for fixed $\varepsilon$ and
    varying $\delta$.}
  \label{fig:bs1}
\end{figure}

\begin{figure}[t]
  \centering
  \includegraphics[width=0.32\textwidth]{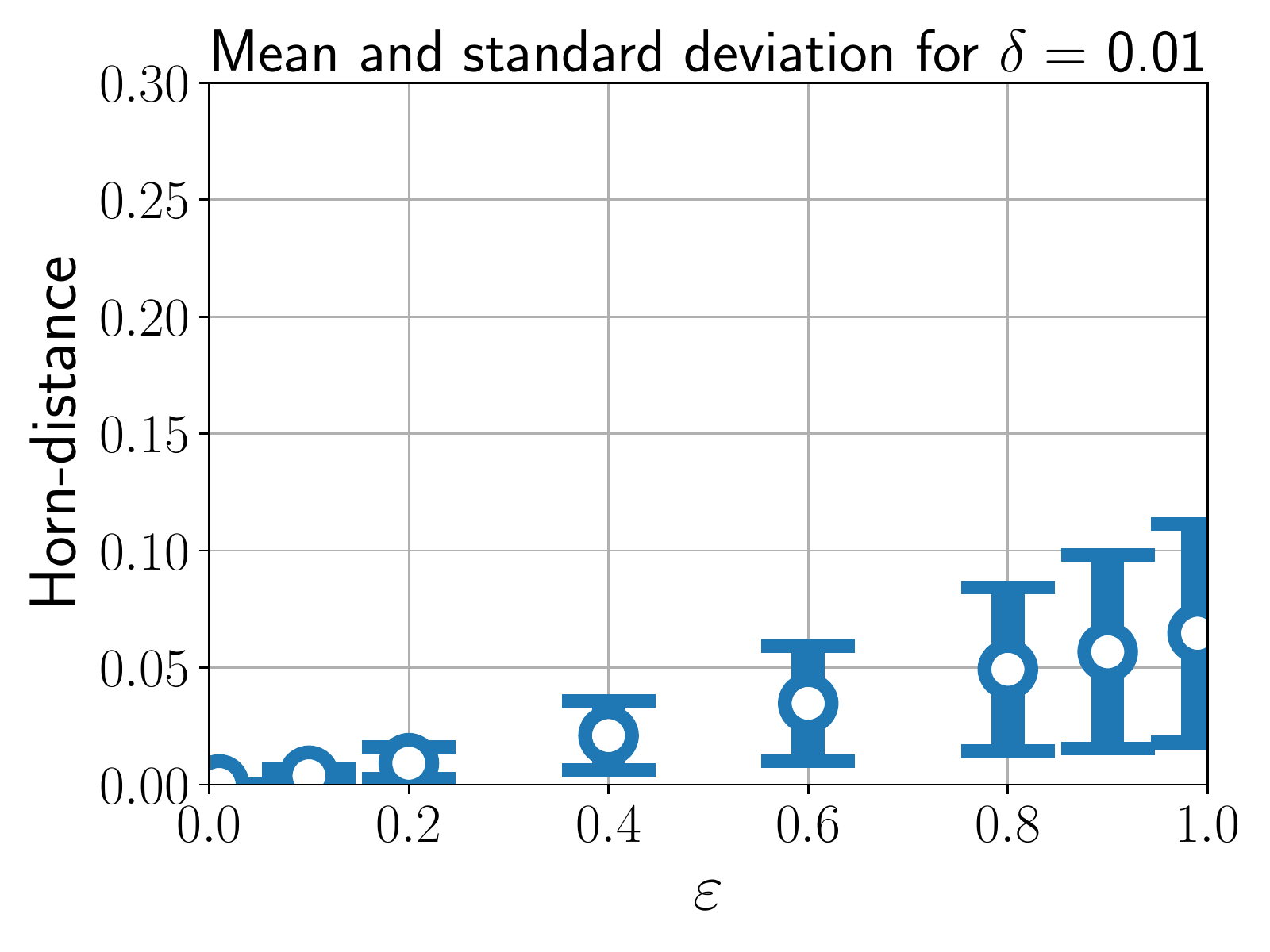}
  \includegraphics[width=0.32\textwidth]{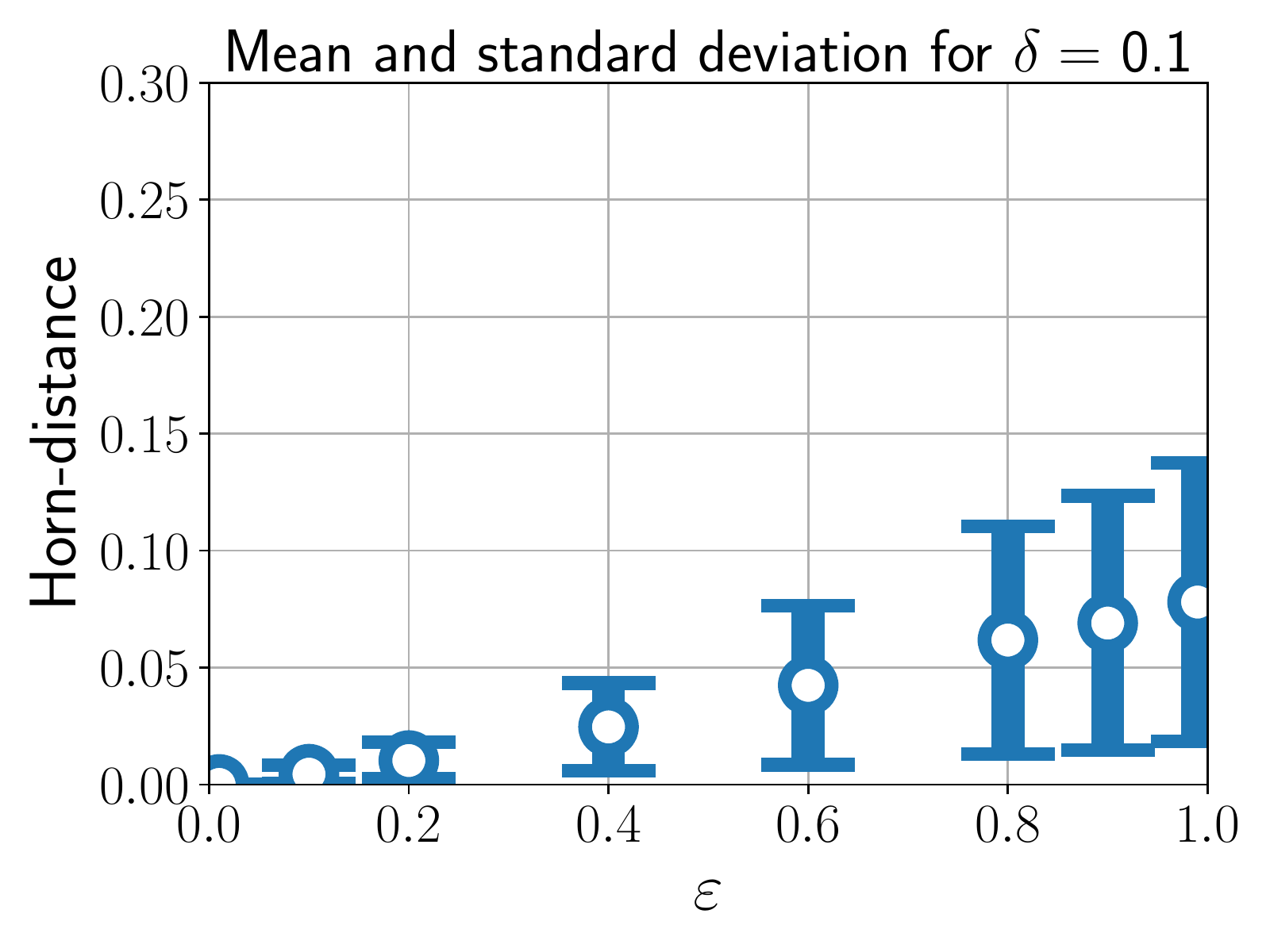}
  \includegraphics[width=0.32\textwidth]{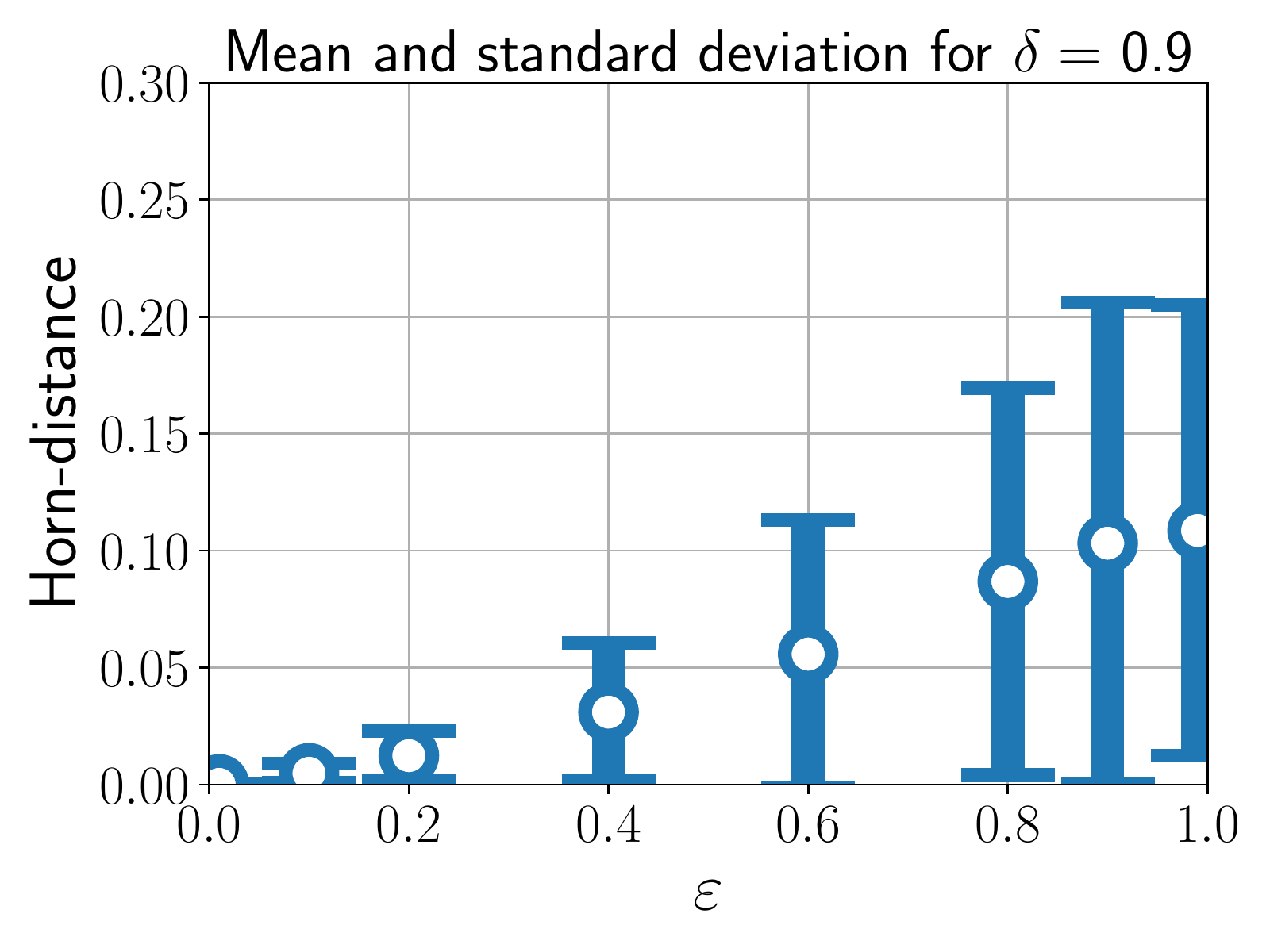}
  \caption{Horn-distance between the contexts from the BibSonomy data
    and corresponding PAC~bases for fixed $\delta$ and varying
    $\varepsilon$.}
  \label{fig:bs2}
\end{figure}

This data set consists of a collection of 2835 formal contexts, each
having exactly 12 attributes. It was created utilizing a data-dump
from the BibSonomy platform, and a detailed description of how this
had been done can be found in~\cite{conf/cla/BorchmannH16}. Those
contexts have a varying number of objects and their canonical bases
have sizes between one and 189.

Let us first fix the accuracy $\varepsilon$ and vary the confidence
$\delta$ in order to investigate the influence of the latter.  The
mean value and the standard deviation over all 2835 formal contexts of
the Horn-distance between the canonical basis and a PAC~basis is shown
in~\cref{fig:bs1}.  A first observation is that for all chosen values
of $\varepsilon$, an increase of $1-\delta$ only yields a small change
of the mean value, in most cases an increase as well. The standard
deviation is, in almost all cases, also increasing. The results for
the macro average of precision and recall are shown
in~\cref{fig:bsprec}.  Again, only a small impact on the final outcome
when varying $1-\delta$ could be observed. We therefore omitted to
show these in favor of the following plots.

Dually, let us now fix the confidence $\delta$ and vary the accuracy
$\varepsilon$.  The Horn-distances between the canonical basis and a
computed PAC~basis for this experiment are shown in~\cref{fig:bs2}.
From this we can learn numerous things.  First, we see that increasing
$\varepsilon$ always leads to a considerable increase in the
Horn-distance, signaling that the PAC basis deviates more and more
from the canonical basis.  However, it is important to note that the
mean values are always below $\varepsilon$, most times even
significantly.  Also, the increase for the Horn-distance while
increasing $\varepsilon$ is significantly smaller than one.  That is
to say, the required accuracy bound is never realized, and especially
for larger values of $\varepsilon$ the deviation of the computed
PAC~basis from the exact implicational theory is less than the
algorithm would allow to.  We observe a similar behavior for precision
and recall.  For small values of $\varepsilon$, both precision and
recall are very high, i.e., close to one, and subsequently seem to
follow an exponential decay.

\begin{figure}[t]
  \centering
  \includegraphics[width=0.32\textwidth]{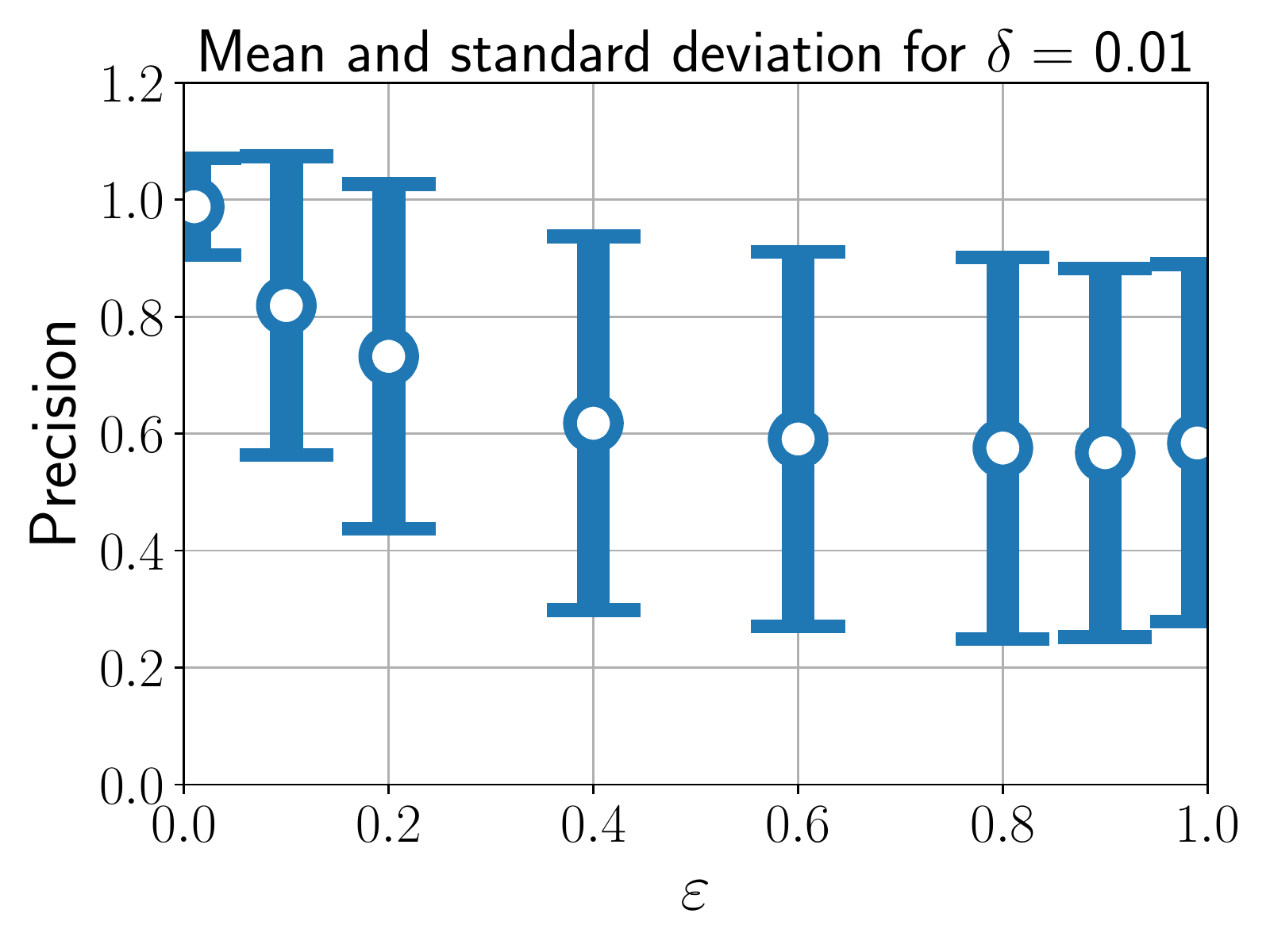}
  \includegraphics[width=0.32\textwidth]{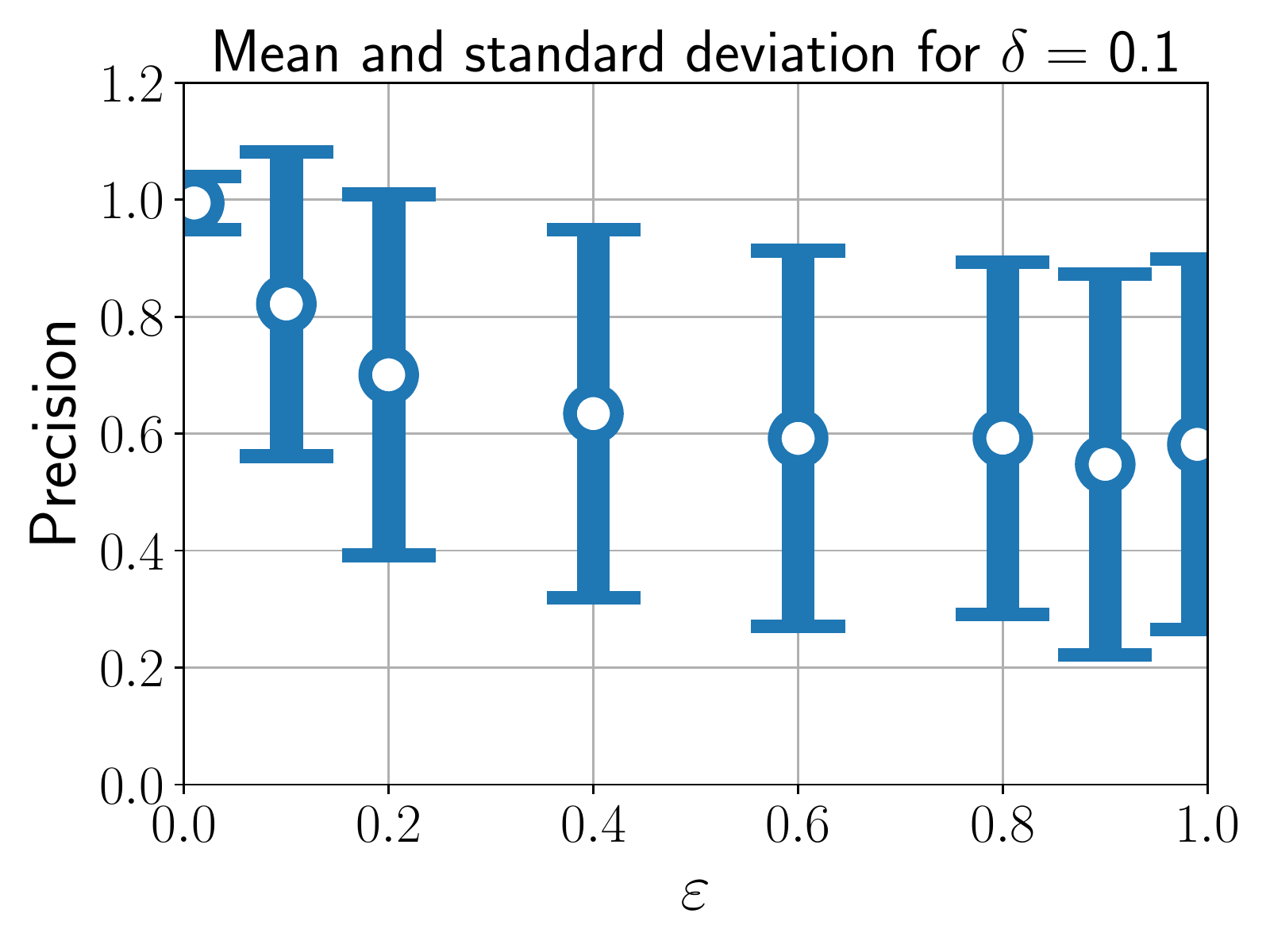}
  \includegraphics[width=0.32\textwidth]{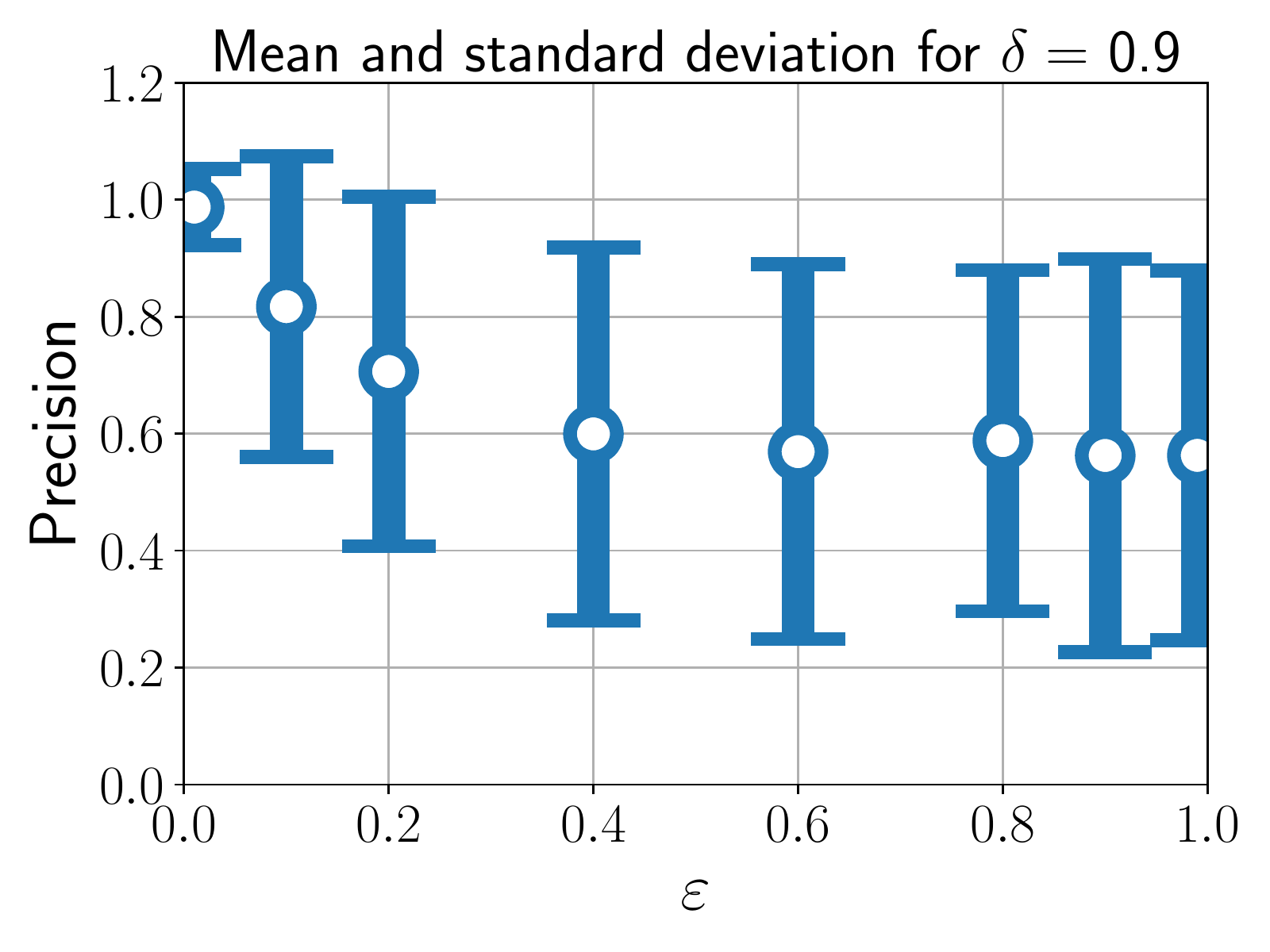}\\
  \includegraphics[width=0.32\textwidth]{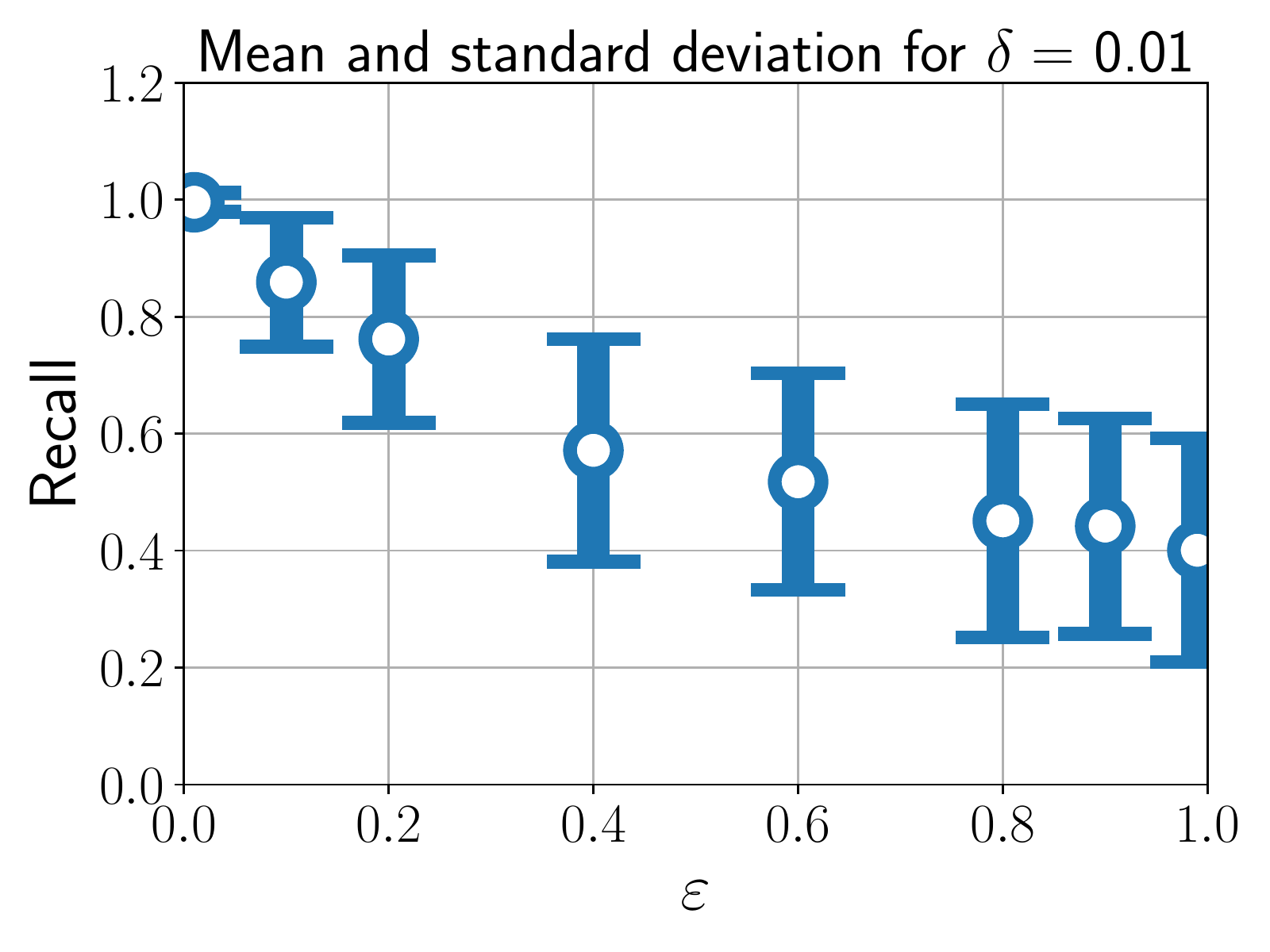}
  \includegraphics[width=0.32\textwidth]{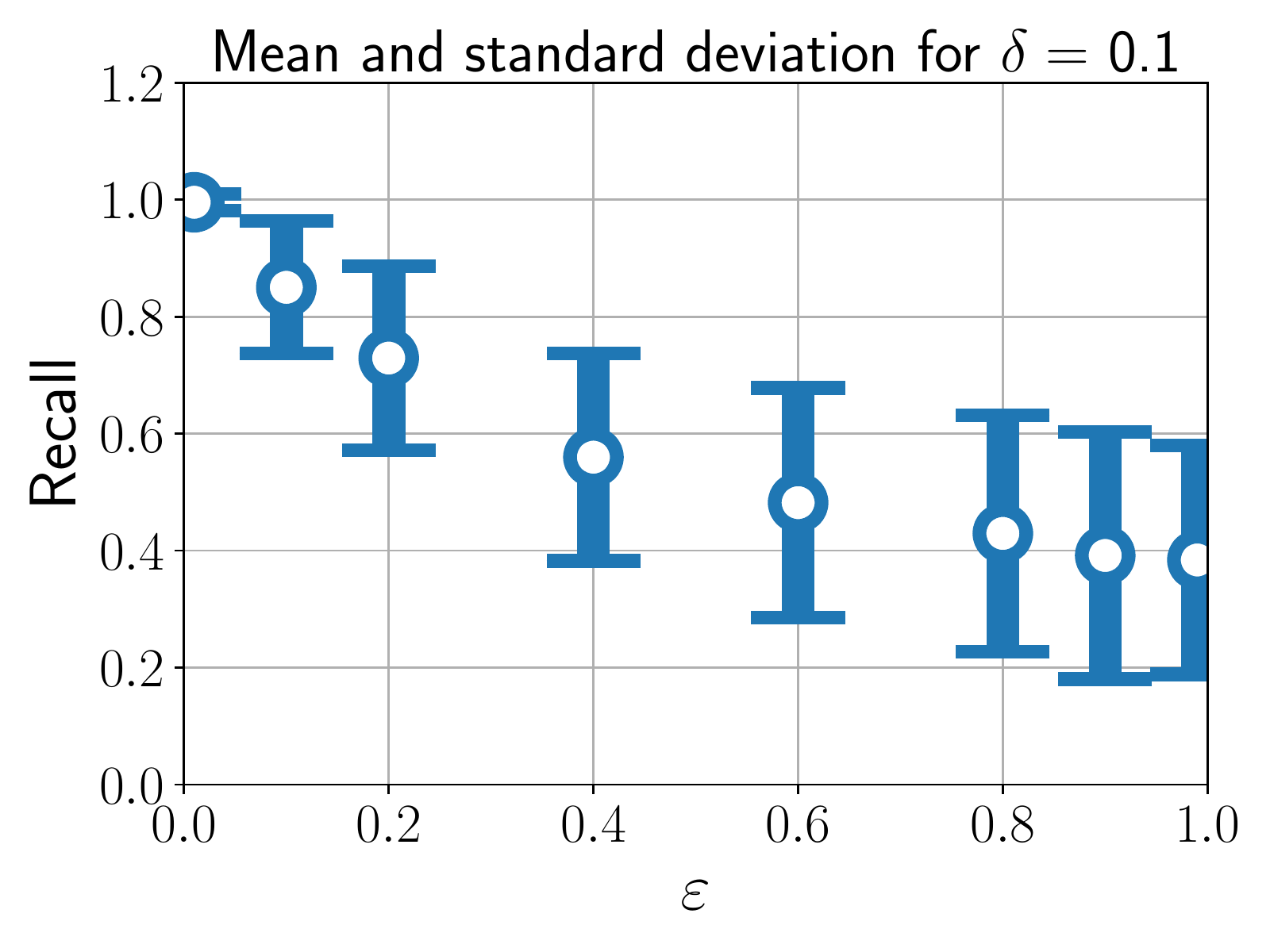}
  \includegraphics[width=0.32\textwidth]{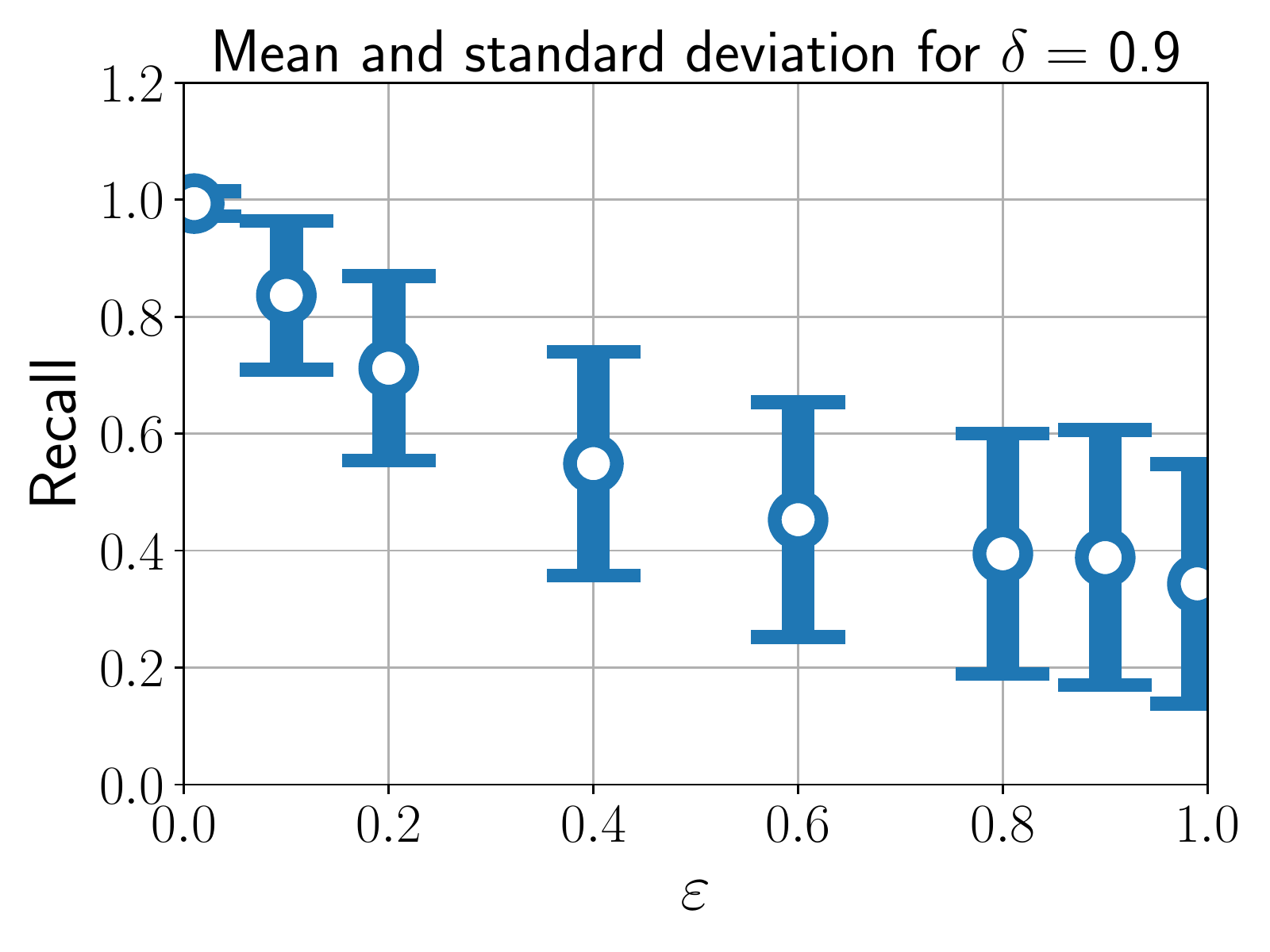}
  \caption{Measured precision (above) and recall (below) for fixed
    $\delta$ and varying $\varepsilon$ for the BibSonomy data set.}
  \label{fig:bsprec}
\end{figure}

% \begin{figure}[t]
%   \centering
%   \includegraphics[width=0.32\textwidth]{{visualize/Recall_plot_0.01}.pdf}
%   \includegraphics[width=0.32\textwidth]{{visualize/Recall_plot_0.1}.pdf}
%   \includegraphics[width=0.32\textwidth]{{visualize/Recall_plot_0.9}.pdf}\\
%   \includegraphics[width=0.32\textwidth]{{visualize/Recall_plot_var_0.01}.pdf}
%   \includegraphics[width=0.32\textwidth]{{visualize/Recall_plot_var_0.1}.pdf}
%   \includegraphics[width=0.32\textwidth]{{visualize/Recall_plot_var_0.9}.pdf}
%   \caption{Measured recall for fixed $\varepsilon$ and varying
%     $\delta$ (above) and fixed $\delta$ and varying $\varepsilon$ (below).}
%   \label{fig:bsreca}
% \end{figure}

\subsubsection{Artificial contexts}
\label{sec:random-contexts}

\begin{figure}[t]
  \centering
  \includegraphics[width=0.32\textwidth]{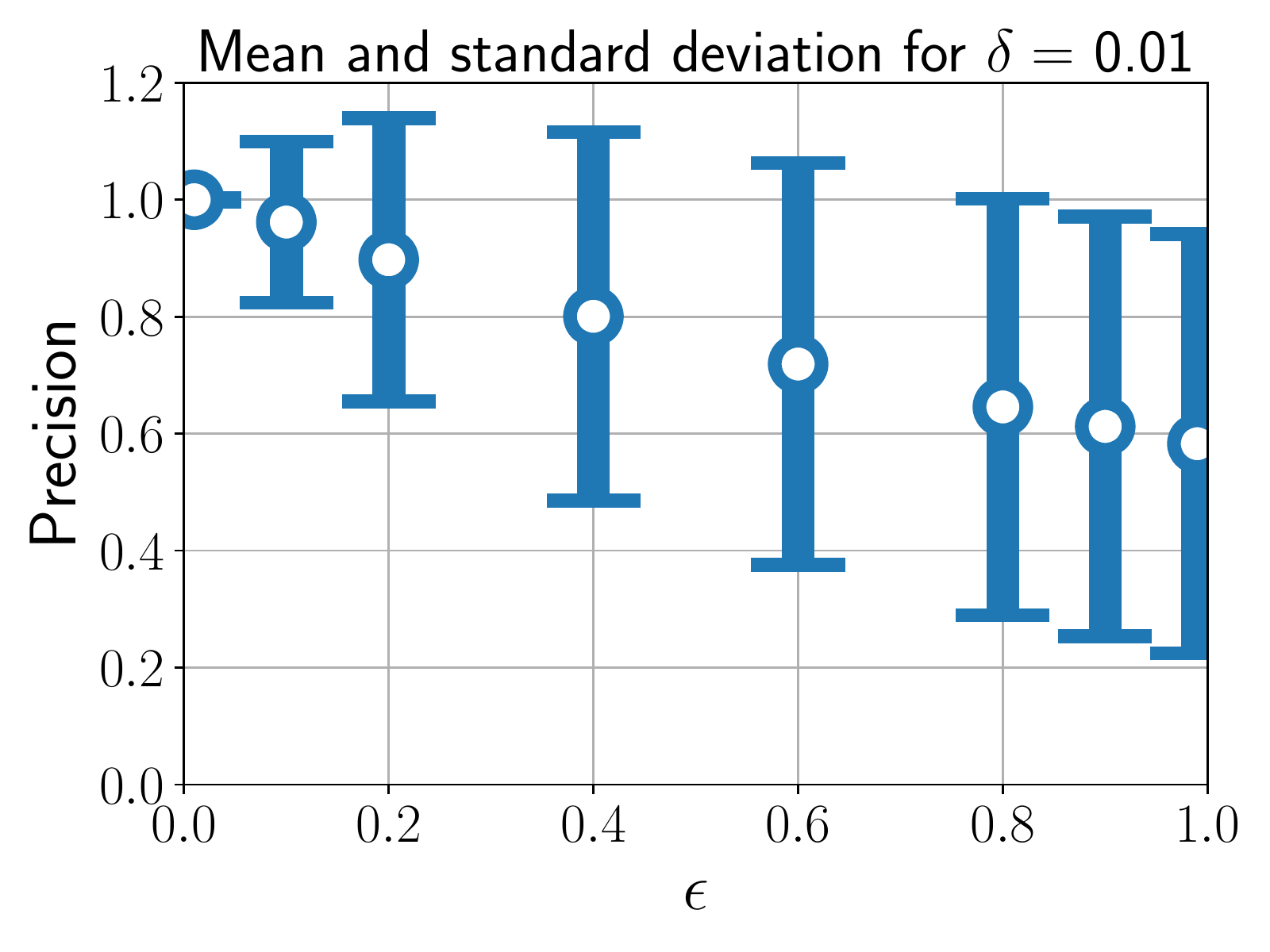}
  \includegraphics[width=0.32\textwidth]{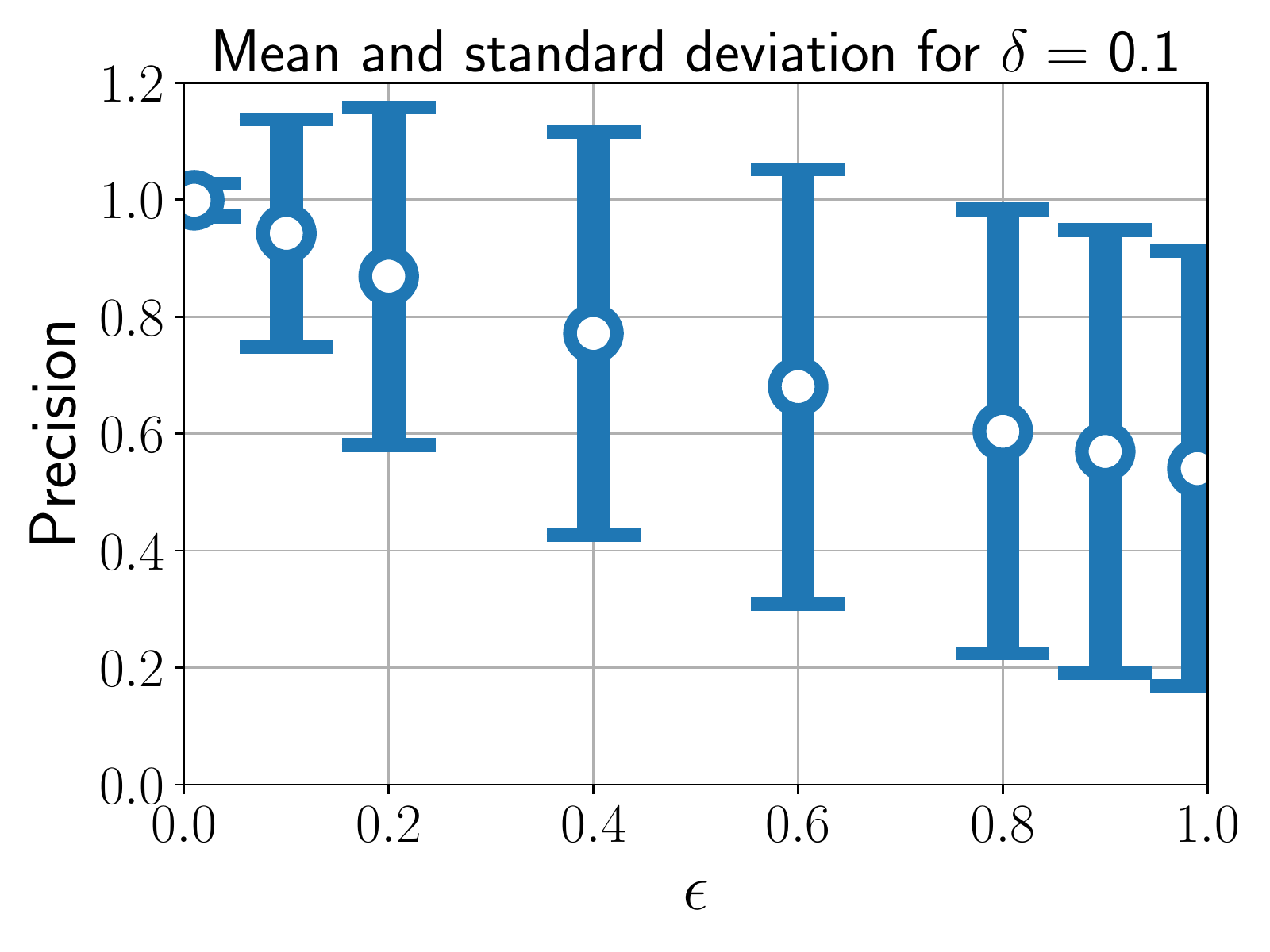}
  \includegraphics[width=0.32\textwidth]{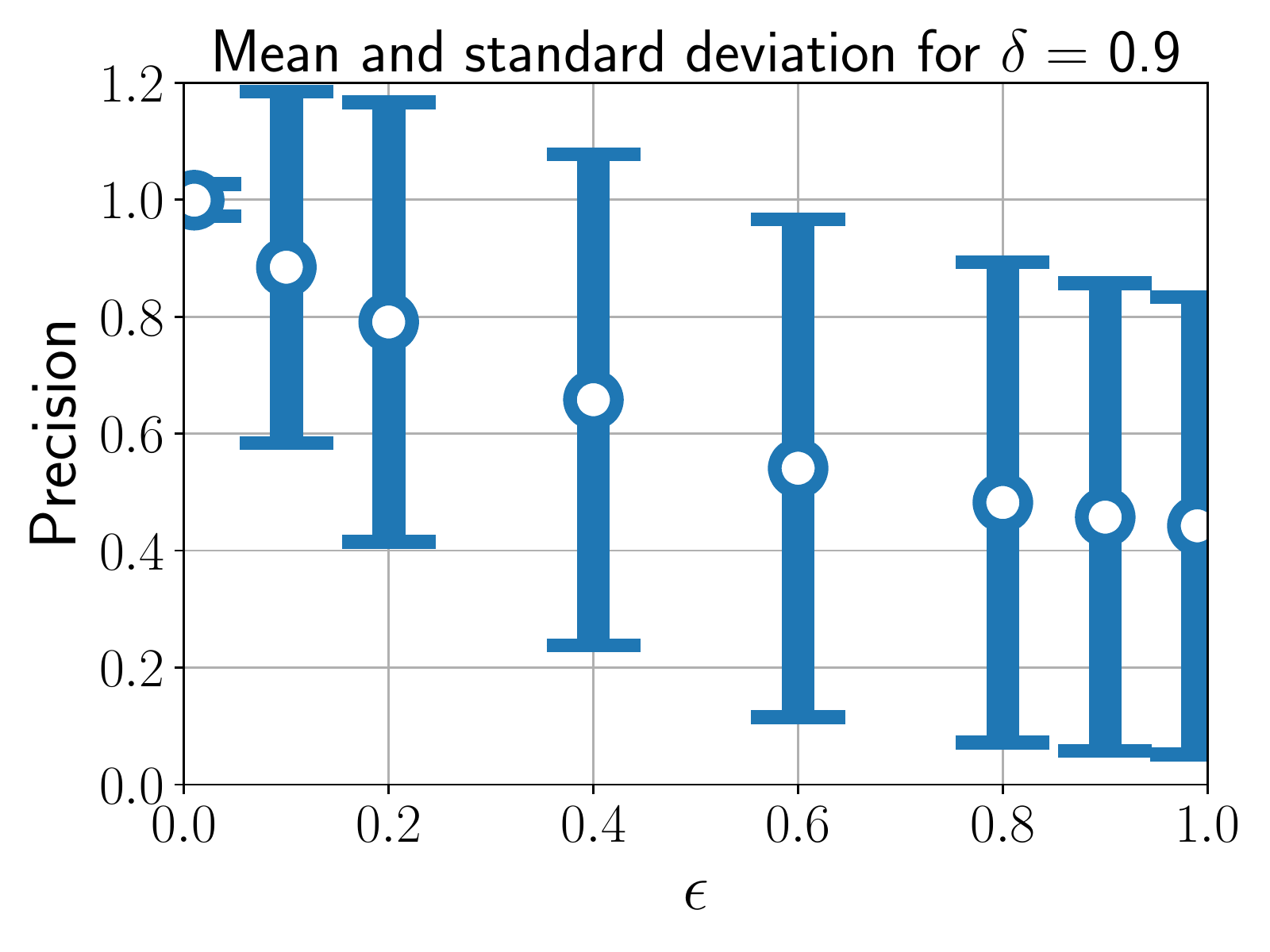}\\
  \includegraphics[width=0.32\textwidth]{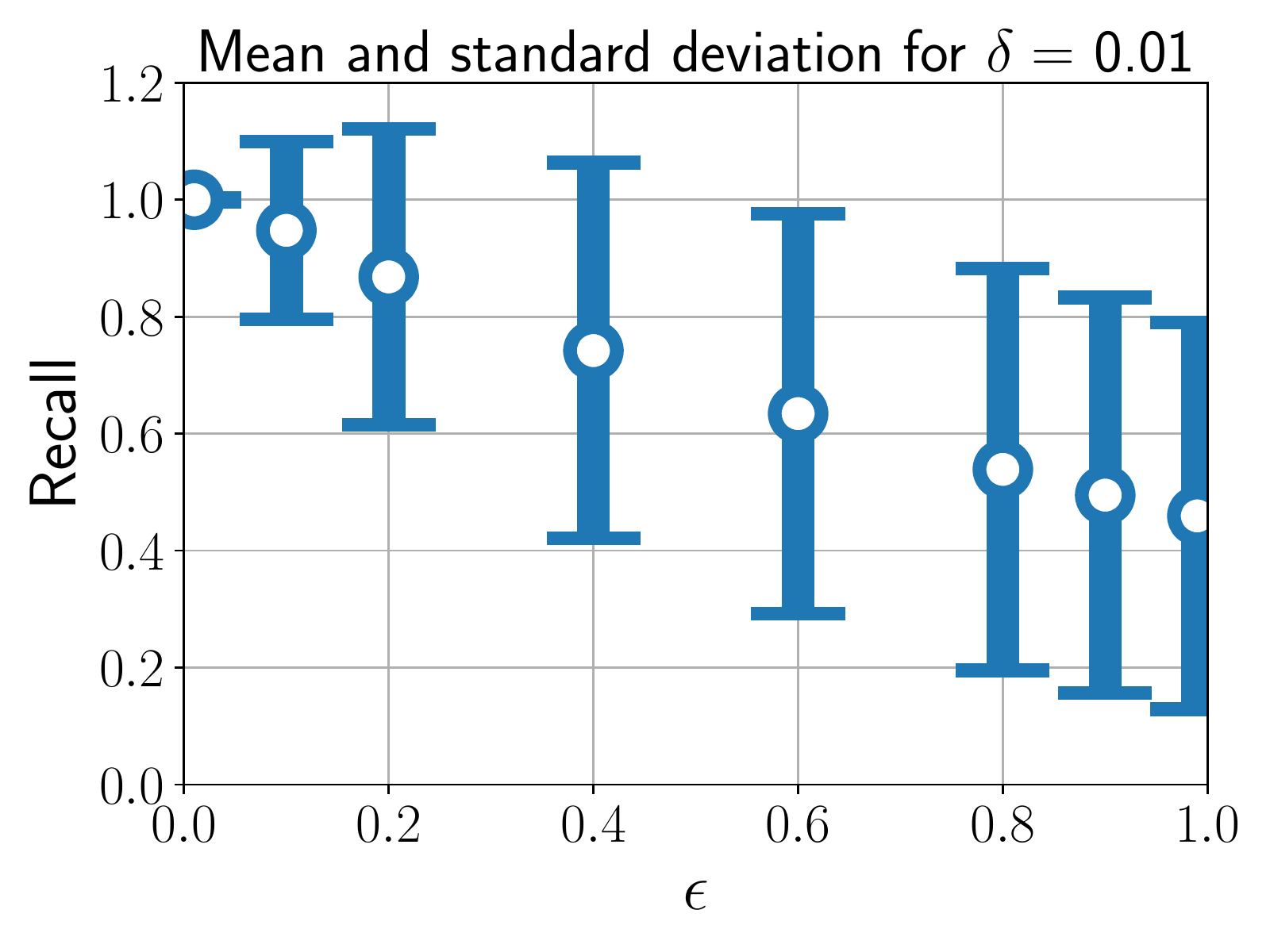}
  \includegraphics[width=0.32\textwidth]{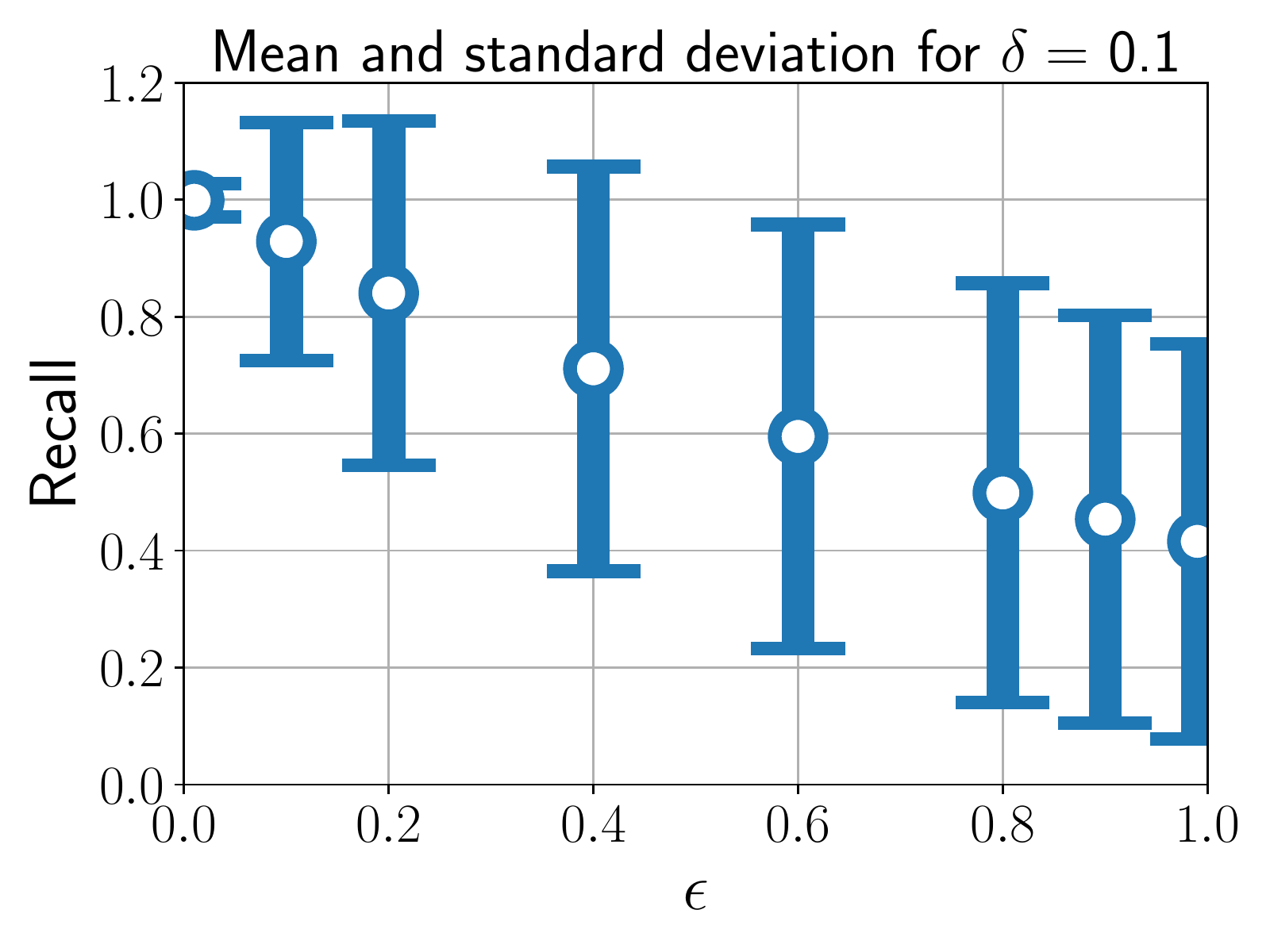}
  \includegraphics[width=0.32\textwidth]{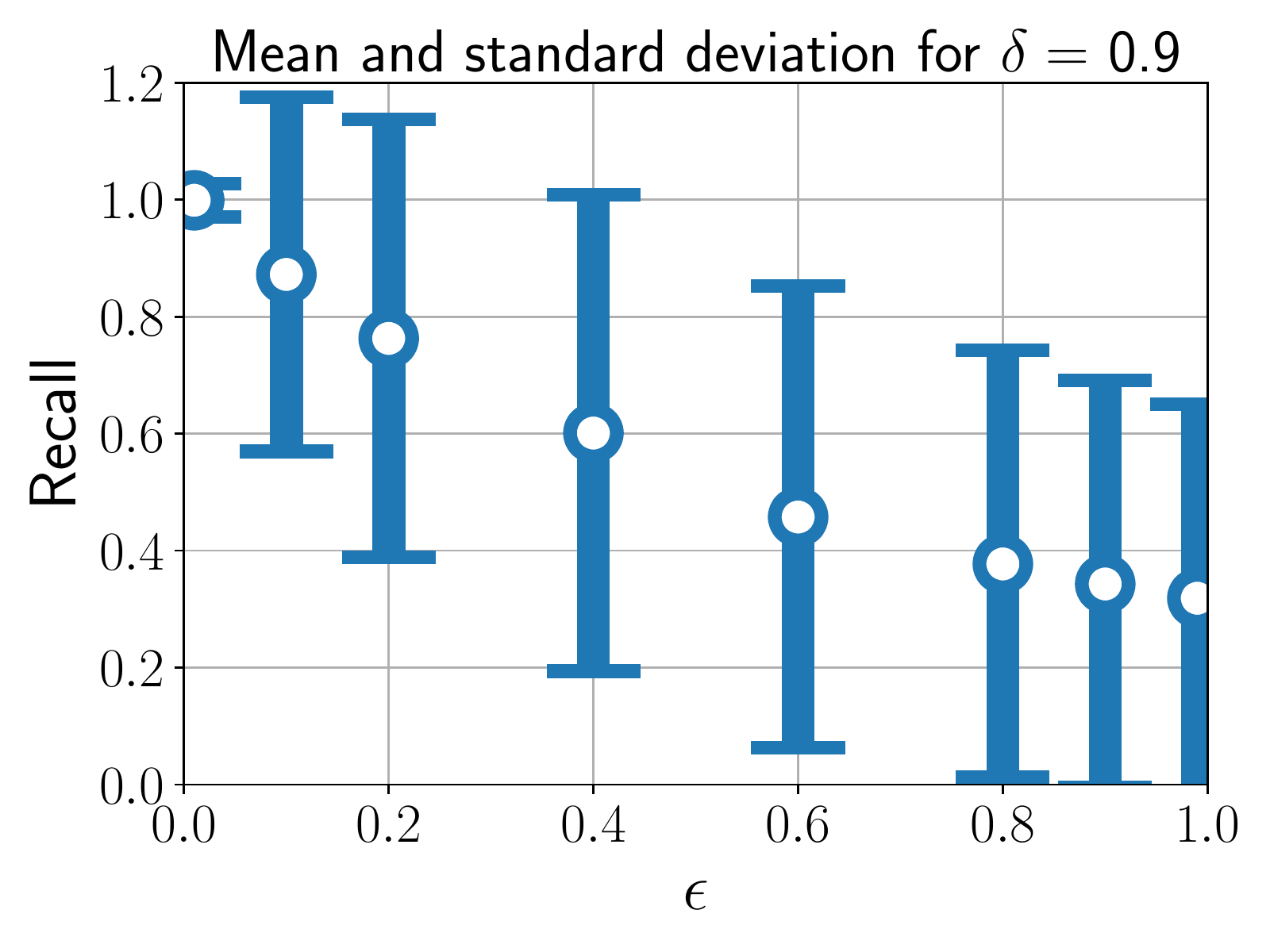}
  \caption{Measured recall for fixed $\varepsilon$ and varying
    $\delta$ (above) and fixed $\delta$ and varying $\varepsilon$
    (below) for 3939 randomly generated formal contexts with ten attributes.}
  \label{fig:random}
\end{figure}

We now want to discuss the results of a computation analogous to the
previous one, but with artificially generated formal contexts. For
these formal contexts, the size of the attribute set is fixed at ten,
and the number of objects and the density are chosen uniformly at
random.  The original data set consists of 4500 formal contexts, but
we omit all that have a canonical basis with fewer than ten
implications, to eliminate the high impact a single false implication
in bases of small cardinality would have.

A selection of the experimental results is shown
in~\cref{fig:random}. We limit the presentation to precision and
recall only, since the previous experiments indicate that
investigating Horn-distance does not yield any new insights.  For
$\varepsilon=0.01$ and $\delta-1 = 0.01$, the precision as well as the
recall is almost exactly one (0.999), with a standard deviation of
almost zero (0.003).  When increasing $\varepsilon$, the mean values
deteriorate analogously to the previous experiment, but the
standard deviation increases significantly more.

\subsubsection{Stability}
\label{sec:stability}
\begin{figure}[t]
  \centering
  \includegraphics[width=0.32\textwidth]{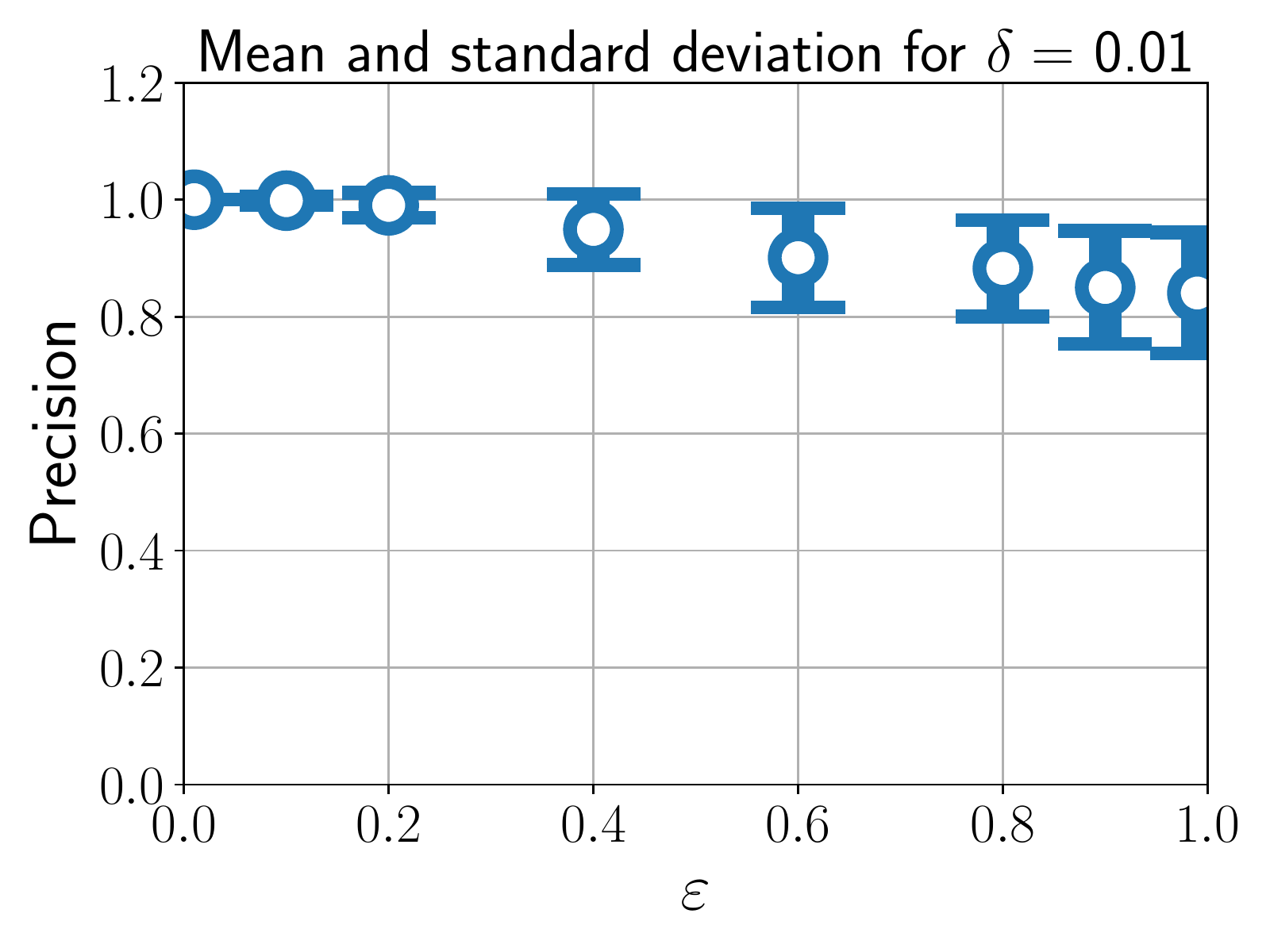}
  \includegraphics[width=0.32\textwidth]{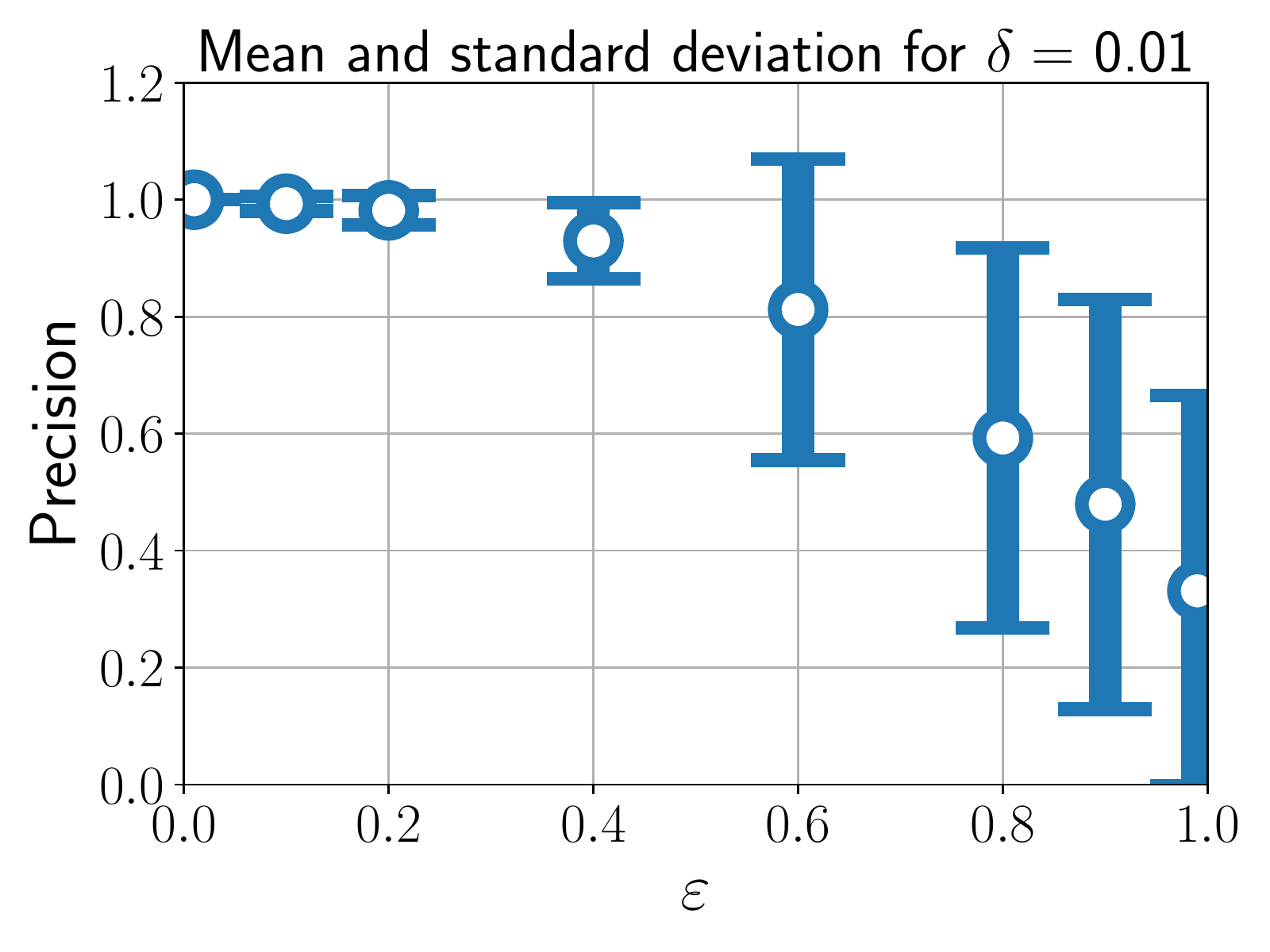}
  \includegraphics[width=0.32\textwidth]{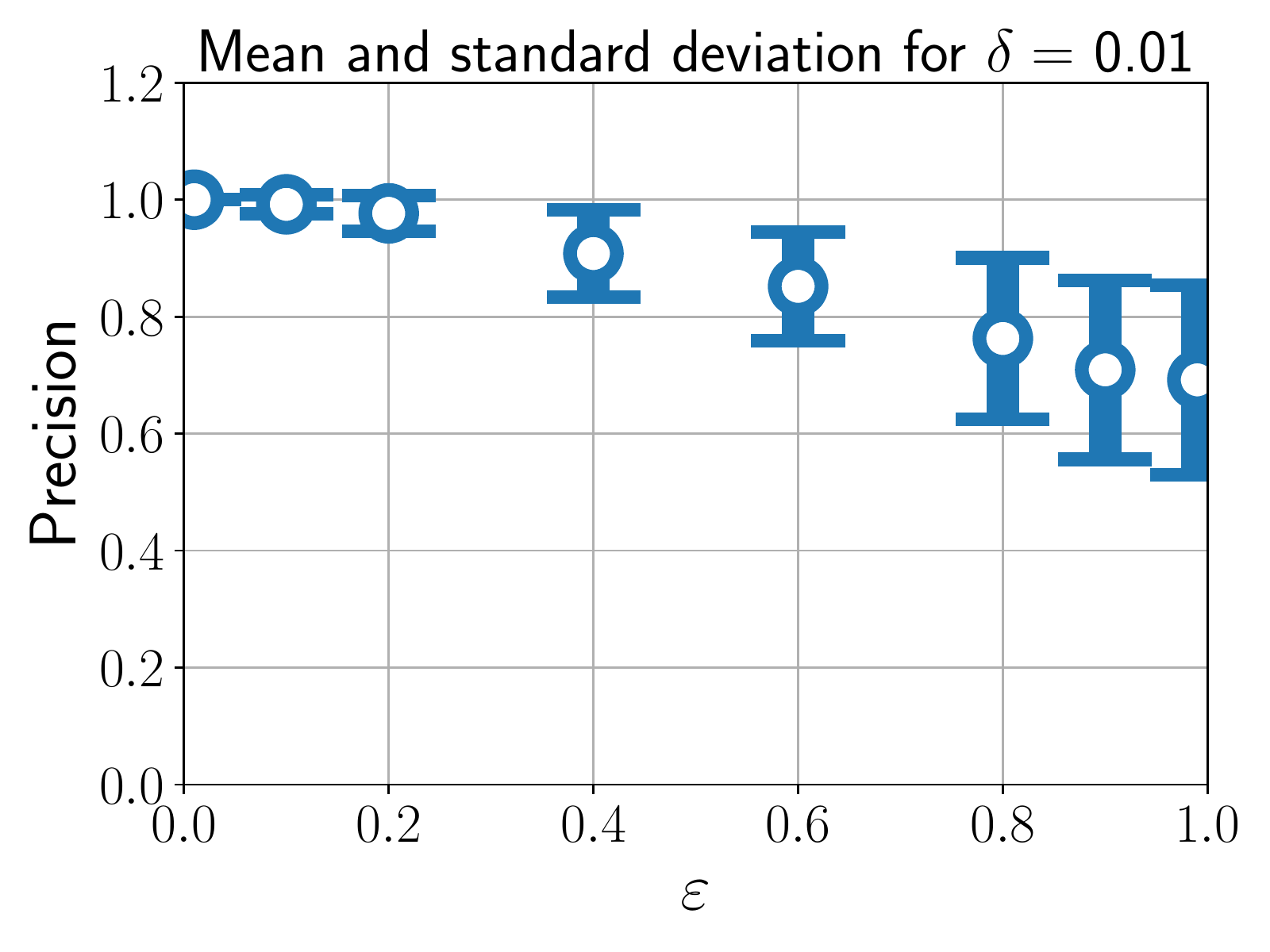}\\
  \includegraphics[width=0.32\textwidth]{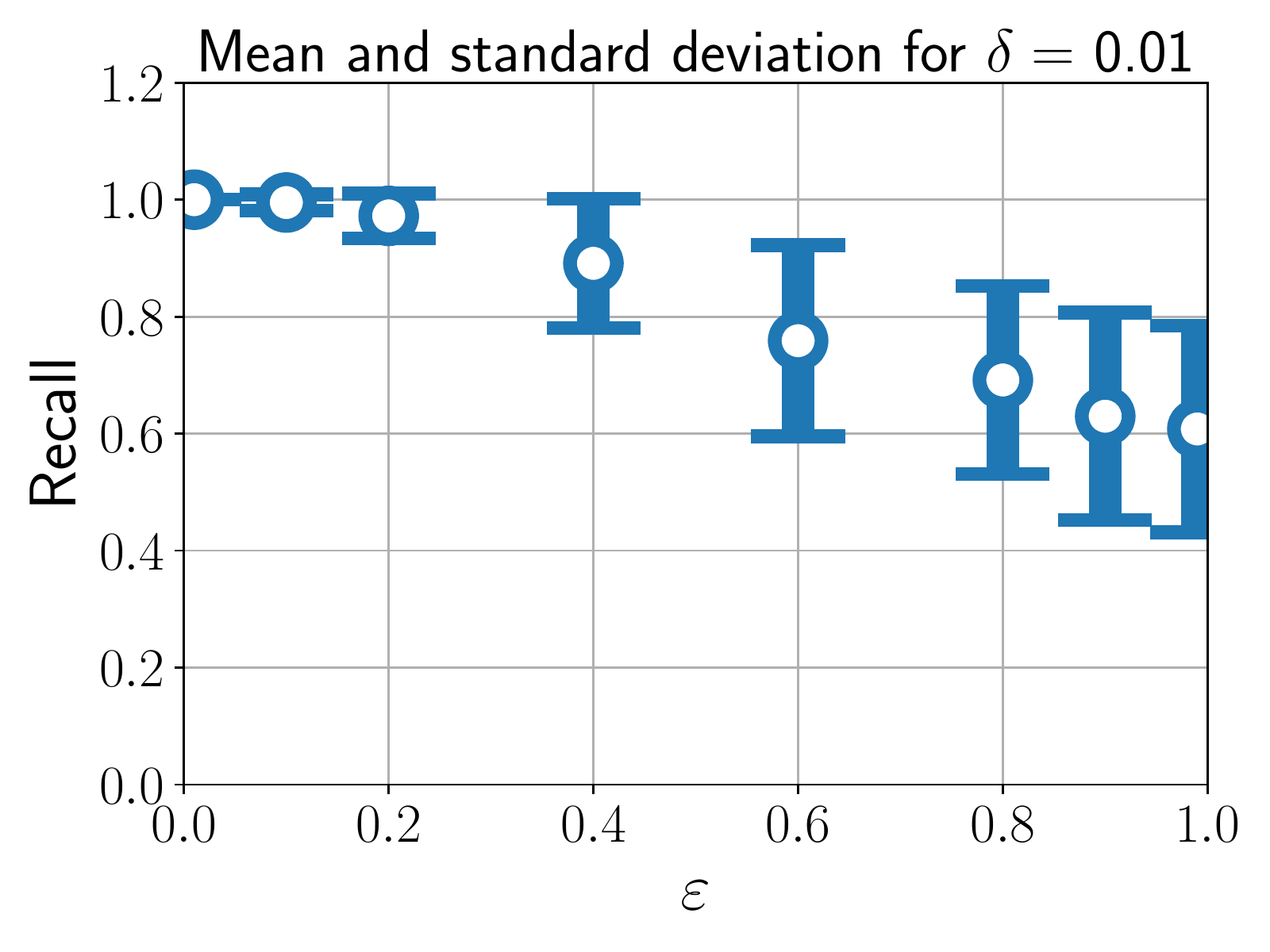}
  \includegraphics[width=0.32\textwidth]{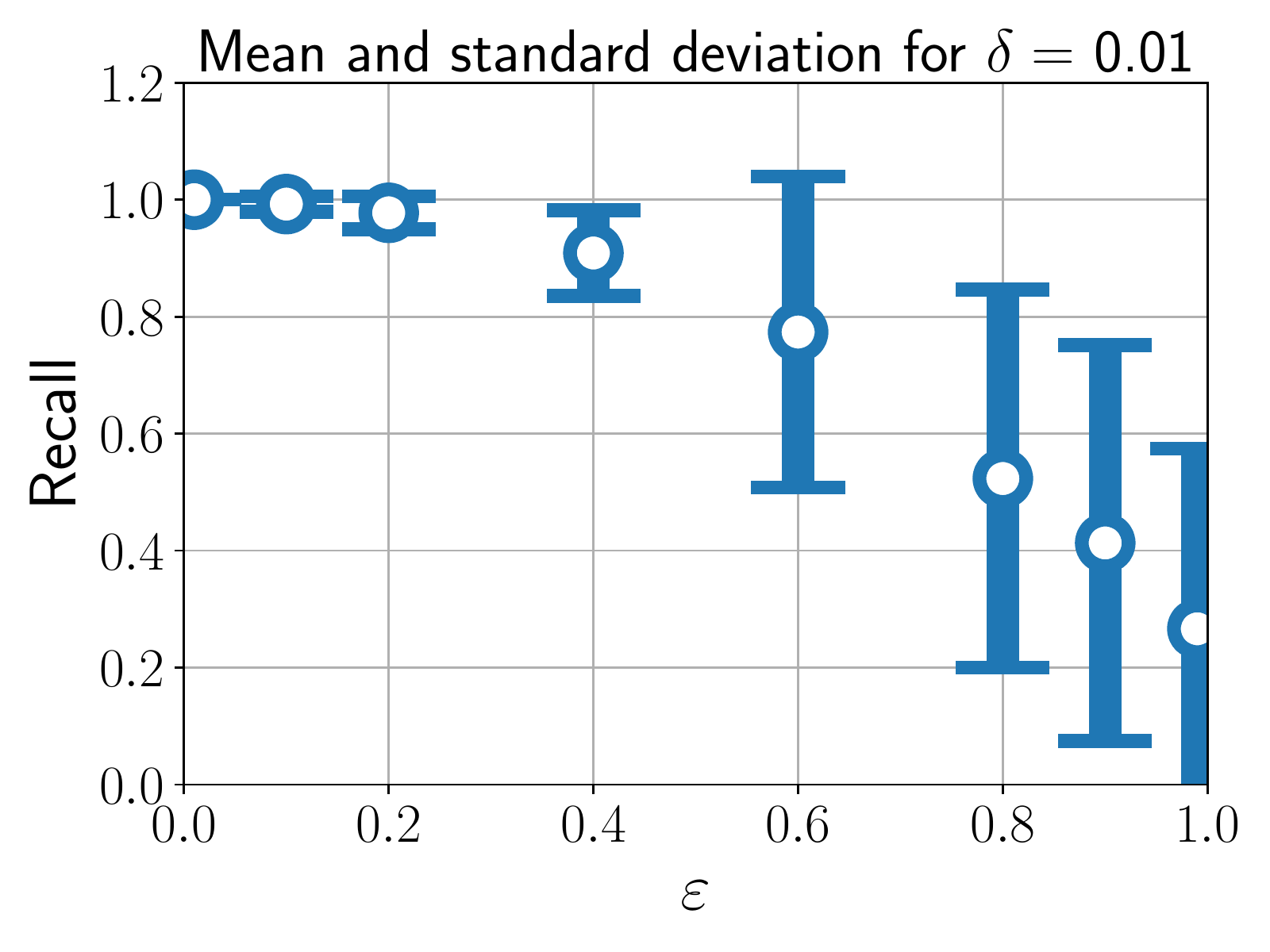}
  \includegraphics[width=0.32\textwidth]{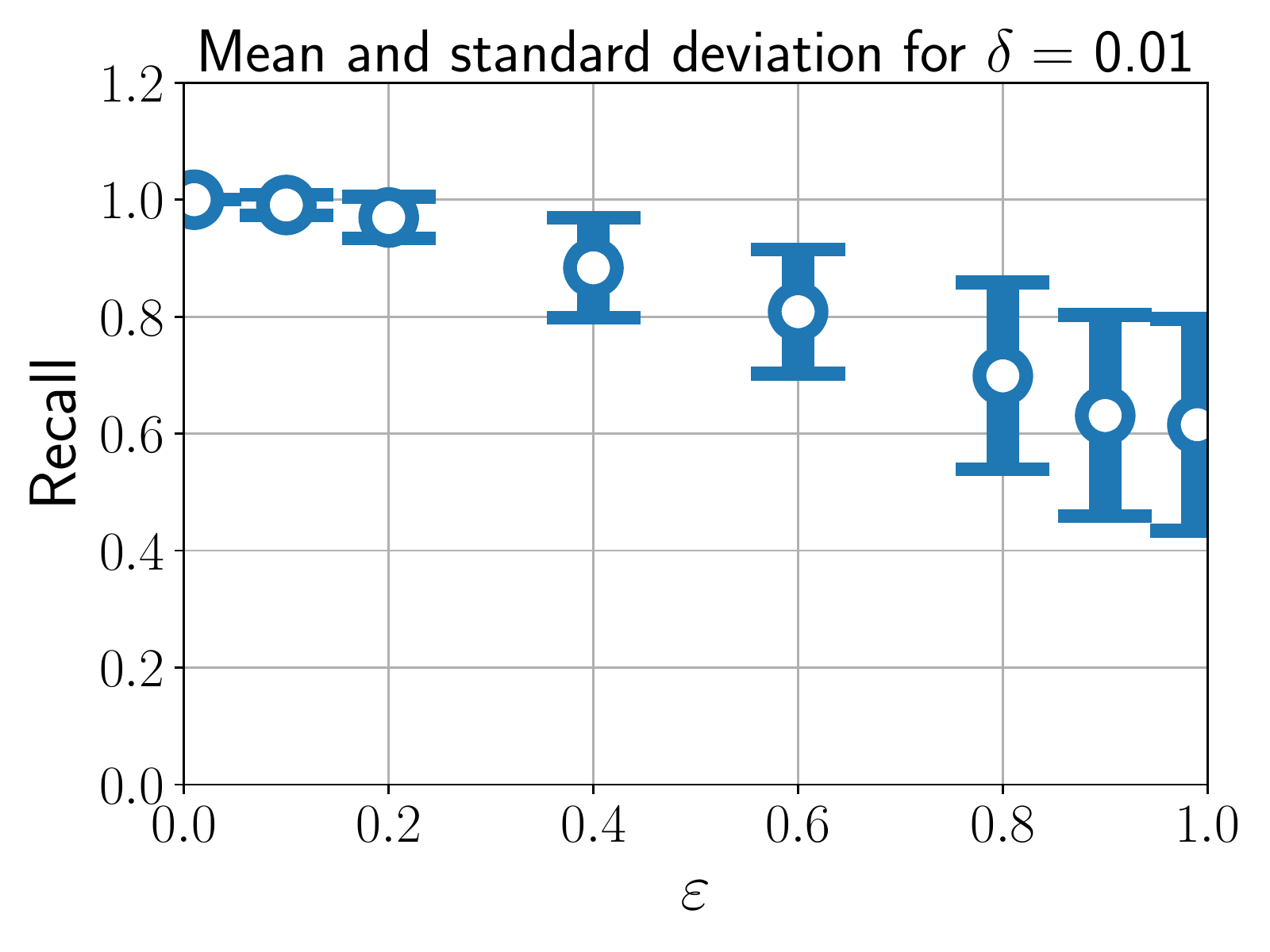}
  \caption{For fixed $\delta$ and varying $\varepsilon$, measured
    precision (above) and recall (below) stability for 100 runs on the
    same formal context with eight (left), nine (middle), and ten
    (right) attributes.}
  \label{fig:stability}
\end{figure}

In our final experiment, we want to consider the impact of the
randomness of the \lstinline{pac-basis} algorithm when computing bases
of fixed formal contexts.  To this end, we shall consider particular
formal contexts $\mathbb K$ and repeatedly compute probably
approximately correct implication bases of $\mathbb K$.  For these
bases, we again compute recall and precision as we did in the previous
experiments.

We shall consider three different artificial formal contexts with
eight, nine, and ten attributes, and canonical bases of size 31, 40,
and 70, respectively. In~\cref{fig:stability}, we show the precision
and recall values for these contexts when calculating PAC~bases
100~times. In general, the standard-deviation of precision and recall
for small values of $\varepsilon$ are low.  Increasing this parameter
leads to an exponential decay of precision and recall, as expected,
and the standard-deviation increases as well.  We expect that both the
decay of the mean value as well as the increase in standard deviation
are less distinct for formal contexts with large canonical bases.

\subsubsection{Discussion}
\label{sec:discussion}

Altogether the experiments show promising results. However, there are
some peculiarities to be discussed. The impact of $1-\delta$ for
Horn-distance in the case of the BibSonomy data set was considerably
low. At this point, it is not clear whether this is due to the nature
of the chosen contexts or to the fact that the algorithm is less
constrained by $\delta$. The results presented in~\cref{fig:bsprec}
show that neither precision nor recall are impacted by varying
$1-\delta$ as well.  All in all, for the formal contexts of the
BibSonomy data set, the algorithm delivered solid results in terms of
accuracy and confidence, in particular when considering precision and
recall, see~\cref{fig:bsprec}. Both measures indicate that the
PAC~bases perform astonishingly well, even for high values of
$\varepsilon$.

For the experiment of the artificial contexts, the standard deviation
increases significantly more than in the BibSonomy experiment. The
source for this could not be determined in this work and needs further
investigation. The overall inferior results for the artificial
contexts, in comparison to the results for the BibSonomy data set, may
be credited to the fact that many of the artificial contexts had a
small canonical basis between 10 and 30. For those, a small amount of
false or missing implications had a great impact on precision and
recall. Nevertheless, the promising results for small values of
$\varepsilon$ back the usability of the PAC basis generating
algorithm.

\subsection{A Small Case-Study}
\label{sec:how-much-different}

Let us consider a classical example, namely the \emph{Star-Alliance}
context~\cite{gerd_airlines}, consisting of the members of the Star
Alliance airline alliance prior to 2002, together with the regions of
the world they fly to.  The formal context $\mathbb{K}_{\mathsf{SA}}$
is given in Figure~\ref{fig:star-alliance}; it consists of
13~airlines and 9~regions, and $\Can(\mathbb K_{\mathsf{SA}})$
consists of 13~implications.

\def\Caribbean{\ensuremath{\mathsf{Caribbean}}}
\def\LatinAmerica{\ensuremath{\mathsf{Latin\ America}}}
\def\UnitedStates{\ensuremath{\mathsf{United\ States}}}
\def\Africa{\ensuremath{\mathsf{Africa}}}
\def\AsiaPacific{\ensuremath{\mathsf{Asia\ Pacific}}}
\def\Europe{\ensuremath{\mathsf{Europe}}}
\def\MiddleEast{\ensuremath{\mathsf{Middle\ East}}}
\def\Canada{\ensuremath{\mathsf{Canada}}}
\def\Mexico{\ensuremath{\mathsf{Mexico}}}

\begin{figure}[tp]
  \centering
  \def\rbox#1{\rotatebox{90}{\text{\strut #1}}}
  \def\x{\times}
  \begin{equation*}
    \begin{array}{l|*{9}{c}}
      ~ & \rbox{\LatinAmerica} & \rbox{\Europe} & \rbox{\Canada}
      & \rbox{\AsiaPacific} & \rbox{\MiddleEast} & \rbox{\Africa}
      & \rbox{\Mexico} & \rbox{\Caribbean} & \rbox{\UnitedStates} \\
      \midrule
      \text{Air Canada} & \x & \x & \x & \x & \x & & \x & \x & \x \\
      \text{Air New Zealand} & & \x & & \x & & & & & \x \\
      \text{All Nippon Airways} & & \x & & \x & & & & & \x\\
      \text{Ansett Australia} & & & & \x & & & & & \\
      \text{The Austrian Airlines Group} & & \x & \x & \x & \x & \x &
                      & & \x\\
      \text{British Midlands} & & \x & & & & & & & \\
      \text{Lufthansa} & \x & \x & \x & \x & \x & \x & \x & & \x\\
      \text{Mexicana} & \x & & \x &&& \x & \x & \x \\
      \text{Scandinavian Airlines} & \x&\x&&\x&&\x&&&\x\\
      \text{Singapore Airlines} & &\x&\x&\x&\x&\x&&&\x\\
      \text{Thai Airways International} & \x&\x&&\x&&&&\x&\x\\
      \text{United Airlines} & \x&\x&\x&\x&&&\x&\x&\x\\
      \text{VARIG} & \x&\x&&\x&&\x&\x&&\x\\
    \end{array}
  \end{equation*}
  \caption{Star-Alliance Context $\mathbb{K}_{\mathsf{SA}}$}
  \label{fig:star-alliance}
\end{figure}

In the following, we shall investigate PAC~bases of
$\mathbb K_{\mathsf{SA}}$ and compare them to
$\Can(\mathbb K_{\mathsf{SA}})$.  Note that due to the probabilistic
nature of this undertaking, it is hard to give certain results, as the
outcomes of \lstinline{pac-basis} can be different on different
invocations, as seen in~\cref{sec:stability}.  It is nevertheless
illuminating to see what results are possible for certain values of
the parameters $\epsilon$ and $\delta$.  In particular, we shall see
that implications returned by \lstinline{pac-basis} are still
meaningful, even if they are not valid in $\mathbb K_{\mathsf{SA}}$.

% (set (approx-canonical-base ctx-star-alliance 0.1 0.1))
% #{(#{Caribbean} ==> #{LatinAmerica UnitedStates})
%   (#{Africa LatinAmerica AsiaPacific Mexico Europe UnitedStates Canada} ==> #{MiddleEast Caribbean})
%   (#{Africa Caribbean LatinAmerica AsiaPacific Europe UnitedStates} ==> #{MiddleEast Mexico Canada})
%   (#{MiddleEast} ==> #{AsiaPacific Europe UnitedStates Canada})
%   (#{AsiaPacific UnitedStates} ==> #{Europe})
%   (#{LatinAmerica} ==> #{UnitedStates})
%   (#{AsiaPacific Europe} ==> #{UnitedStates})
%   (#{Canada} ==> #{UnitedStates})
%   (#{LatinAmerica UnitedStates Canada} ==> #{Mexico})
%   (#{Europe UnitedStates} ==> #{AsiaPacific})
%   (#{Caribbean LatinAmerica Mexico UnitedStates} ==> #{Canada})
%   (#{Africa} ==> #{AsiaPacific Europe UnitedStates})
%   (#{Mexico} ==> #{LatinAmerica UnitedStates})}

As a first case, let us consider comparably small values of accuracy
and confidence, namely $\epsilon = 0.1$ and a $\delta = 0.1$.  For
those values we obtained a basis $\mathcal{H}_{0.1,0.1}$ that differs
from $\Can(\mathbb K_{\mathsf{SA}})$ only in the implication
\begin{equation*}
  \Africa, \AsiaPacific, \Europe, \UnitedStates, \Canada \to \MiddleEast
\end{equation*}
being replaced by
\begin{equation}
  \label{eq:2}
  \Africa, \LatinAmerica, \AsiaPacific, \Mexico, \Europe,
  \UnitedStates, \Canada \to \bot
\end{equation}
Indeed, for the second implication to be refuted by the algorithm, the
only counterexample in $\mathbb K_{\mathsf{SA}}$ would have been
Lufthansa, which does not fly to the Caribbean.  However, in our
particular run of \lstinline{pac-basis} that produced
$\mathcal{H}_{0.1,0.1}$, this counterexample had not been considered,
resulting in the implication from Equation~\eqref{eq:2} to remain in
the final basis.  Thus, while $\mathcal{H}_{0.1,0.1}$ does not
coincide with $\Can(\mathbb K_{\mathsf{SA}})$, the only implication in
which they differ~\eqref{eq:2} still has very high \emph{confidence}
in $\mathbb K_{\mathsf{SA}}$, in the sense of the usual notions of
support and confidence of association
rules~\cite{arules:agrawal:association-rules}.\todo{this is not true:
  the implication has confidence zero; the implication is nevertheless
  \enquote{almost correct}, but not in the sense of its confidence}
Therefore, the basis $\mathcal{H}_{0.1,0.1}$ can be considered as a
good approximation of $\Can(\mathbb K_{\mathsf{SA}})$.

% (set (approx-canonical-base ctx-star-alliance 0.1 0.05))
% #{(#{Caribbean} ==> #{LatinAmerica UnitedStates})
%   (#{Africa Caribbean LatinAmerica AsiaPacific Europe UnitedStates} ==> #{MiddleEast Mexico Canada})
%   (#{MiddleEast} ==> #{AsiaPacific Europe UnitedStates Canada})
%   (#{AsiaPacific UnitedStates} ==> #{Europe})
%   (#{LatinAmerica} ==> #{UnitedStates})
%   (#{AsiaPacific Europe} ==> #{UnitedStates})
%   (#{Canada} ==> #{UnitedStates})
%   (#{LatinAmerica UnitedStates Canada} ==> #{Mexico})
%   (#{Europe UnitedStates} ==> #{AsiaPacific})
%   (#{Caribbean LatinAmerica Mexico UnitedStates} ==> #{Canada})
%   (#{Africa} ==> #{AsiaPacific Europe UnitedStates})
%   (#{Mexico} ==> #{LatinAmerica UnitedStates})
%   (#{Africa AsiaPacific Europe UnitedStates Canada} ==> #{MiddleEast LatinAmerica Mexico})}

% (set (approx-canonical-base ctx-star-alliance 0.1 0.8))
% #{(#{Caribbean} ==> #{LatinAmerica UnitedStates})
%   (#{Africa Caribbean LatinAmerica AsiaPacific Europe UnitedStates} ==> #{MiddleEast Mexico Canada})
%   (#{MiddleEast} ==> #{AsiaPacific Europe UnitedStates Canada})
%   (#{Africa AsiaPacific Europe UnitedStates Canada} ==> #{MiddleEast})
%   (#{AsiaPacific UnitedStates} ==> #{Europe})
%   (#{LatinAmerica} ==> #{UnitedStates})
%   (#{AsiaPacific Europe} ==> #{UnitedStates})
%   (#{Canada} ==> #{UnitedStates})
%   (#{LatinAmerica UnitedStates Canada} ==> #{Mexico})
%   (#{Europe UnitedStates} ==> #{AsiaPacific})
%   (#{Caribbean LatinAmerica Mexico UnitedStates} ==> #{Canada})
%   (#{Africa} ==> #{AsiaPacific Europe UnitedStates})
%   (#{Mexico} ==> #{LatinAmerica UnitedStates})}

As in the previous section, it turns out that increasing the parameter
$\delta$ to values larger than $0.1$ does not change much of resulting
basis.  This is to be expected, since $\delta$ is a bound on the
probability that the basis returned by \lstinline{pac-basis} is not of
accuracy $\epsilon$.  Indeed, even for as large a value as $\delta =
0.8$, the resulting basis we obtained in our run of
\lstinline{pac-basis} was exactly $\Can(\mathbb K_{\mathsf{SA}})$.
Nevertheless, care must be exercised when increasing $\delta$, as this
increases the chance that \lstinline{pac-basis} returns a basis that is
far off from the actual canonical basis -- if not in this run, then
maybe in a latter one.

Conversely to this, and in accordance to the results of the previous
section, increasing $\epsilon$, and thus decreasing the bound on the
accuracy, does indeed have a notable impact on the resulting basis.  For
example, for $\epsilon = 0.5$ and $\delta = 0.1$, our run of
\lstinline{pac-basis} returned the basis
\begin{gather*}
  ({\Caribbean} \to \bot),\ ({\AsiaPacific, \Mexico} \to \bot),\
  ({\AsiaPacific, \Europe} \to \bot),\\
  ({\MiddleEast} \to \bot),\ ({\LatinAmerica} \to {\Mexico, \UnitedStates, \Canada}).
\end{gather*}
While this basis enjoys a small Horn-distance to $\Can(\mathbb
K_{\mathsf{SA}})$ of around 0.11, it can hardly be considered usable,
as it ignores a great deal of objects in $\mathbb K_{\mathsf{SA}}$.
Changing the confidence parameter $\delta$ to smaller or larger values
again did not change much of the appearance of the bases.

To summarize, for our example context $\mathbb K_{\mathsf{SA}}$, we
have seen that low values of $\epsilon$ often yield bases that are
very close to the canonical basis of $\mathbb K_{\mathsf{SA}}$, both
intuitively and in terms of Horn-distance to the canonical basis of
$\mathbb K_{\mathsf{SA}}$.  However, the larger the values of
$\epsilon$ get, the less useful bases returned by
\lstinline{pac-basis} appear to be.  On the other hand, varying the
value for the confidence parameter $\delta$ within certain reasonable
bounds does not seem to influence the results of \lstinline{pac-basis}
very much.

\section{Summary and Outlook}
\label{sec:summary-outlook}

The goal of this work is to give first evidence that probably
approximately correct implication bases are a practical substitute for
their exact counterparts, possessing advantageous algorithmic
properties.  To this end, we have argued both quantitatively and
qualitatively that PAC~bases are indeed close approximations of the
canonical basis of both artificially generated as well as real-world
data sets.  Moreover, the fact that PAC~bases can be computed in
output-polynomial time alleviates the usual long running times of
algorithms computing implication bases, and renders the applicability
on larger data sets possible.

To push forward the usability of PAC~bases, more studies are
necessary.  Further investigating the quality of those bases on
real-world data sets is only one concern.  An aspect not considered in
this work is the \emph{actual} running time necessary to compute
PAC~bases, compared to the one for the canonical basis, say.  To make
such a comparison meaningful, a careful implementation of the
\lstinline{pac-basis} algorithm needs to be devised, taking into
account aspects of algorithmic design that are beyond the scope of
this work.

We also have not considered relationships between PAC~bases and
existing ideas for extracting implicational knowledge from data.  For
example, in our investigation of Section~\ref{sec:how-much-different},
it turned out that implications extracted by the algorithm enjoy a
high confidence in the data set.  One could conjecture that there is a
deeper connection between PAC~bases and the notions of support and
confidence of implications.  It is also not too far fetched to imagine
a notion of PAC~bases that incorporates support and confidence right
from the beginning.

The classical algorithm to compute the canonical basis of a formal
context can easily be extended to the algorithm of \emph{attribute
  exploration}.  This algorithm, akin to query learning, aims at
finding an exact representation of an implication theory that is only
accessible through a \emph{domain expert}.  As the algorithm for
computing the canonical basis can be extended to attribute exploration,
we are certain that it is also possible to extend the
\lstinline{pac-basis} algorithm to a form of \emph{probably
  approximately correct attribute exploration}.  Such an algorithm,
while not being entirely exact, would be highly sufficient for the
inherently erroneous process of learning knowledge from human experts,
while possibly being much faster.  On top of that, existing work in
query learning handling non-omniscient, erroneous, or even malicious
oracles could be extended to attribute exploration so that it could
deal with erroneous or malicious domain experts.  In this way,
attribute exploration could be made much more robust for learning
tasks in the world wide web.

\medskip

\noindent\textit{Acknowledgments}: Daniel Borchmann gratefully
acknowledges support by the Cluster of Excellence “Center for
Advancing Electronics Dresden” (cfAED). The computations presented in
this paper were conducted by conexp-clj, a general purpose software
for formal concept analysis
(\url{https://github.com/exot/conexp-clj}).

\appendix

\printbibliography

\end{document}